\documentclass[nohyperref]{article}
\usepackage[accepted]{icml2022}
\usepackage{times}
\usepackage{amsmath,amsfonts,bm}
\usepackage{xcolor,hyperref}
\hypersetup{
    colorlinks = true,
    linkbordercolor = {blue},
    citecolor = {blue}
}
\usepackage{enumitem}
\usepackage{authblk}    
\usepackage[utf8]{inputenc} 
\usepackage[T1]{fontenc}    

\allowdisplaybreaks

\newcommand{\dist}{\mathrm{dist}}
\newcommand{\prox}[1]{\mathrm{prox}_{#1}}

\newcommand{\dom}{\mathop{\mathrm{dom}}}
\newcommand{\zero}{\mathbf{0}}

\newcommand{\KL}{K{\L}~}

\newcommand{\inner}[2]{\langle #1, #2 \rangle}
\newcommand{\red}[1]{\color{red} #1 \color{black}}

\DeclareMathOperator*{\argmin}{arg\,min}

\usepackage{url}            
\usepackage{booktabs}       
\usepackage{amsfonts}       
\usepackage{microtype}      
\usepackage{algorithm}
\usepackage{algorithmic}
\usepackage{amsthm} 
\usepackage{thm-restate}
\usepackage[capitalize,noabbrev]{cleveref}

\usepackage{nicefrac}       
\usepackage{environ}

\makeatletter

\newtheorem{lemma}{Lemma} 
\newtheorem{assum}{Assumption} 
 
\newtheorem{definition}{Definition}

\usepackage{colortbl}
\usepackage{multirow}

\usepackage{threeparttable}
\usepackage{hhline}
\usepackage{makecell}
\usepackage{subcaption}
\usepackage{graphicx}
\usepackage{grffile}  

\usepackage[textsize=tiny]{todonotes}

\usepackage[toc,page,header]{appendix}
\usepackage{minitoc}

\allowdisplaybreaks[4]   
\makeatletter
{
	\begin{center}
		\ref stepcounter{algorithm}
		\hrule height.8pt depth0pt \kern2pt
		\renewcommand{\caption}[2][\relax]{
			{\raggedright\textbf{\ALG@name~\thealgorithm} ##2\par}%
			\ifx\relax##1\relax 
			\addcontentsline{loa}{algorithm}{\protect\numberline{\thealgorithm}##2}%
			\else 
			\addcontentsline{loa}{algorithm}{\protect\numberline{\thealgorithm}##1}%
			\fi
			\kern2pt\hrule\kern2pt
		}
	}{
		\kern2pt\hrule\relax
	\end{center}
}
\makeatother

\begin{document}
\title{A Fast and Convergent Proximal Algorithm for Regularized Nonconvex and Nonsmooth Bi-level Optimization}


\author[1]{\textit{Ziyi Chen}}
\author[2]{\textit{Bhavya Kailkhura}}
\author[1]{\textit{Yi Zhou}}

\affil[1]{Department of Electrical and Computer Engineering, University of Utah, Salt Lake City, UT, US}
\affil[2]{Lawrence Livermore National Lab, Livermore, CA, US}

\affil[1]{\small {Email: \{u1276972,yi.zhou\}@utah.edu}}
\affil[2]{\small {Email: kailkhura1@llnl.gov}}
\maketitle
\thispagestyle{empty} 
 
\doparttoc 
\faketableofcontents 

\begin{abstract}
Many important machine learning applications involve regularized nonconvex bi-level optimization. However, the existing gradient-based bi-level optimization algorithms cannot handle nonconvex or nonsmooth regularizers, and they suffer from a high computation complexity in 
nonconvex bi-level optimization. In this work, we study a proximal gradient-type algorithm that adopts the approximate implicit differentiation (AID) scheme for nonconvex bi-level optimization with possibly nonconvex and nonsmooth regularizers. In particular, the algorithm applies the Nesterov's momentum to accelerate the computation of the implicit gradient involved in AID. We provide a comprehensive analysis of the global convergence properties of this algorithm through identifying its intrinsic potential function. In particular, we formally establish the convergence of the model parameters to a critical point of the bi-level problem, and obtain an improved computation complexity $\mathcal{O}(\kappa^{3.5}\epsilon^{-2})$ over the state-of-the-art result. 
Moreover, we analyze the asymptotic convergence rates of this algorithm under a class of local nonconvex geometries characterized by a {\L}ojasiewicz-type gradient inequality. Experiment on hyper-parameter optimization demonstrates the effectiveness of our algorithm.
\end{abstract}

\section{Introduction}
Bi-level optimization has become an important and popular optimization framework that covers a variety of emerging machine learning applications, e.g., meta-learning~\citep{franceschi2018bilevel,bertinetto2018meta,rajeswaran2019meta,ji2020convergence}, hyperparameter optimization~\citep{franceschi2018bilevel,shaban2019truncated,feurer2019hyperparameter}, reinforcement learning~\citep{konda2000actor,hong2020two}, etc. A standard formulation of bi-level optimization takes the following form. 
\begin{align}
\min_{x\in\mathbb{R}^d} f(x, y^*(x)), \quad \mbox{where} \quad y^*(x)= \argmin_{y\in\mathbb{R}^{p}} g(x,y), \nonumber
\end{align}
where the upper- and lower-level objective functions $f$ and $g$ are both jointly continuously differentiable. To elaborate, bi-level optimization aims to minimize the upper-level compositional objective function $f(x,y^*(x))$, in which $y^*(x)$ is the minimizer of the lower-level objective function $g(x,y)$. 

Solving the above bi-level optimization problem is highly non-trivial as the problem involves two nested minimization problems. In the existing literature, many algorithms have been developed for bi-level optimization. In the early works, \cite{hansen1992new,shi2005extended,moore2010bilevel} reformulated the bi-level problem into a single-level problem with constraints on the optimality conditions of the lower-level problem, yet this reformulation involves a large number of constraints that are hard to address in practice. More recently, gradient-based bi-level optimization algorithms have been developed, which leverage either the approximate implicit differentiation (AID) scheme~\citep{domke2012generic,pedregosa2016hyperparameter,gould2016differentiating,liao2018reviving,ghadimi2018approximation,grazzi2020iteration,lorraine2020optimizing} or the iterative differentiation (ITD) scheme~\citep{domke2012generic,maclaurin2015gradient,franceschi2017forward,franceschi2018bilevel,shaban2019truncated,grazzi2020iteration} to estimate the gradient of the upper-level function. In particular, the AID scheme is more popular due to its simplicity and computation efficiency. Specifically, bi-level optimization algorithm with AID (referred to as BiO-AID) has been analyzed for (strongly)-convex upper- and lower-level functions \citep{liu2020generic}, which do not cover bi-level problems in modern machine learning that usually involve nonconvex upper-level objective functions. On the other hand, recent studies have analyzed the convergence of BiO-AID with nonconvex upper-level function and strongly-convex lower-level function, and established the convergence of a certain type of gradient norm to zero \citep{ji2020bilevel,ghadimi2018approximation,hong2020two}. 

However, the existing gradient-based nonconvex bi-level optimization algorithms have limitations in several perspectives. First, they are not applicable to bi-level problems that involve possibly nonsmooth and nonconvex regularizers. For example, in the application of data hyper-cleaning, one can improve the learning performance by adding a nonsmooth and nonconvex regularizer to push the weights of the clean samples towards 1 while push those of the contaminated samples towards 0 (see Section \ref{sec: experiment} for more details). Second, the convergence guarantees of these algorithms typically ensure a weak gradient norm convergence, which does not necessarily imply the desired convergence of the model parameters. Furthermore, these algorithms suffer from a high computation complexity in nonconvex bi-level optimization. 
\textbf{The overarching goal of this work} is to develop an efficient and convergent proximal-type algorithm for solving regularized nonconvex and nonsmooth bi-level optimization problems and address the above important issues. We summarize our contributions as follows.

\subsection{Our Contributions}

We propose a proximal BiO-AIDm algorithm (see \Cref{alg:main_deter}) and study its convergence properties. This algorithm is a proximal variant of the BiO-AID algorithm for solving the following class of regularized nonsmooth and nonconvex bi-level optimization problems.
\begin{align*}
&\min_{x\in\mathbb{R}^d} f(x, y^*(x)) + h(x), \\ \quad\mbox{where} \quad &y^*(x)= \argmin_{y\in\mathbb{R}^{p}} g(x,y),
\end{align*}
where the upper-level objective function $f$ is nonconvex, the lower-level objective function $g$ is strongly-convex for any fixed $x$, and the regularizer $h$ is possibly nonsmooth and nonconvex. In particular, our algorithm applies the Nesterov's momentum to accelerate the computation of the implicit gradient involved in the AID scheme.

We first analyze the global convergence properties of proximal BiO-AIDm under standard Lipschitz and smoothness assumptions on the objective functions. The key to our analysis is to show that proximal BiO-AID admits an intrinsic potential function $H(x_k,y_k)$ that takes the form 
\begin{align}
	H(x,y'):= \Phi(x) +h(x) +\frac{7}{8}\|y^T(x,y')-y^*(x)\|^2, \nonumber
\end{align}
where $y^T(x,y')$ is obtained by applying the Nesterov's accelerated gradient descent to minimize {$g(x,\cdot)$ with initial point $y'$} for $T$ iterations. In particular, we prove that such a potential function is monotonically decreasing along the optimization trajectory, i.e., $H(x_{k+1},y_{k+1}) < H(x_k,y_k)$, which implies that proximal BiO-AIDm can be viewed as a descent-type algorithm and is numerically stable. Based on this property, we formally prove that every limit point of the model parameter trajectory $\{x_k\}_k$ generated by proximal BiO-AIDm is a critical point of the regularized bi-level problem. Furthermore, when the regularizer is convex, we show that proximal BiO-AIDm requires a computation complexity of $\widetilde{\mathcal{O}}(\kappa^{3.5}\epsilon^{-2})$ (number of gradient, Hessian-vector product and proximal evaluations) for achieving a critical point $x$ that satisfies $\|G(x)\| \le \epsilon$, where $\kappa$ denotes the problem condition number and $G(x)$ denotes the proximal gradient mapping.  
This is the first global convergence and complexity result of proximal BiO-AIDm in regularized nonsmooth and nonconvex bi-level optimization, and it improves the state-of-the-art complexity of BiO-AID (for smooth nonconvex bi-level optimization) by a factor of ${\widetilde{\mathcal{O}}}(\sqrt{\kappa})$.

Besides investigating the global convergence properties, we further establish the asymptotic function value convergence rates of proximal BiO-AIDm under a local {\L}ojasiewicz-type nonconvex geometry, which covers a broad spectrum of local nonconvex geometries. Specifically, we characterize the asymptotic convergence rates of proximal BiO-AIDm in the full spectrum of the {\L}ojasiewicz geometry parameter $\theta$. We prove that as the local geometry becomes sharper (i.e., with a larger $\theta$), the asymptotic convergence rate of proximal BiO-AIDm boosts from sublinear convergence to superlinear convergence.

\subsection{Related Work}
{\bf Bi-level Optimization Algorithms.} Bi-level optimization has been studied for decades~\citep{bracken1973mathematical}, and various types of bi-level algorithms have been proposed, including but not limited to single-level penalized methods~\citep{shi2005extended,moore2010bilevel}, and gradient-based methods via AID or ITD-based hypergradient estimation~\citep{domke2012generic,pedregosa2016hyperparameter,franceschi2018bilevel,ghadimi2018approximation,hong2020two,liu2020generic,li2020improved,grazzi2020iteration,ji2020bilevel,lorraine2020optimizing,ji2021lower}. \cite{huang2021enhanced} proposed a Bregman distance-based method. In particular, \citep{ghadimi2018approximation,hong2020two,ji2020bilevel,yang2021provably,chen2021single,guo2021randomized,huang2021enhanced} characterized the complexity analysis for their proposed methods for bi-level optimization problem under different types of loss geometries. \citep{ji2021lower} studied the lower complexity bounds for bi-level optimization under (strongly) convex geometry and proposed a nearly-optimal accelerated algorithm. All the existing analysis of nonconvex bi-level optimization algorithms focuses on the gradient norm convergence. 
In this paper, we formally establish the parameter {and function value} convergence of proximal BiO-AID in regularized nonconvex and nonsmooth bi-level optimization.



{\bf Applications of Bi-level Optimization.} Bi-level optimization has been widely applied to meta-learning ~\citep{snell2017prototypical,franceschi2018bilevel,rajeswaran2019meta,zugner2019adversarial,ji2020multi,ji2021bilevel}, hyperparameter optimization~\citep{franceschi2017forward,shaban2019truncated}, reinforcement learning~\citep{konda2000actor,hong2020two}, and data poisoning~\citep{mehra2020robust}.
 For example,
 \cite{snell2017prototypical} reformulated the meta-learning objective function under a shared embedding model into a bi-level optimization problem. \cite{rajeswaran2019meta} proposed a bi-level optimizer named iMAML as an efficient variant of model-agnostic meta-learning (MAML)~\citep{finn2017model}, and analyzed the convergence of iMAML under the strongly-convex inner-loop loss. \cite{fallah2020convergence} characterized the convergence of MAML and first-order MAML under nonconvex loss functions. \cite{ji2020convergence} studied the convergence behaviors of almost no inner loop (ANIL)~\citep{raghu2019rapid} under different inner-loop loss geometries {of the MAML objective} function. Recently \cite{mehra2020robust} devised bilevel optimization based data poisoning attacks on certifiably robust classifiers.

{\bf Nonconvex Kurdyka-{\L}ojasiewicz Geometry.} A broad class of regular functions has been shown to satisfy the local nonconvex \KL geometry \citep{Bolte2007}, which affects the asymptotic convergence rates of gradient-based optimization algorithms. The \KL geometry has been exploited to study the convergence of various first-order algorithms for solving minimization problems, including gradient descent \citep{Attouch2009},  alternating gradient descent \citep{Bolte2014}, distributed gradient descent \citep{Zhou2016,Zhou_2017a}, accelerated gradient descent \citep{Li2017}. It has also been exploited to study the convergence of second-order algorithms such as Newton's method \citep{Noll2013,Frankel2015} and cubic regularization method \citep{zhou2018convergence}.

\section{Problem Formulation and Preliminaries}\label{sec:alg}
In this paper, we consider the following regularized nonconvex bi-level optimization problem:
\begin{align*}
&\min_{x\in\mathbb{R}^d} f(x, y^*(x)) + h(x), \tag{P}\\ \quad\mbox{where} \quad &y^*(x)= \argmin_{y\in\mathbb{R}^{p}} g(x,y),
\end{align*}
where both the upper-level objective function $f$ and the lower-level objective function $g$ are jointly continuously differentiable, and the regularizer $h$ is possibly nonsmooth and nonconvex. We note that adding a regularizer to the bi-level optimization problem allows us to impose desired structures on the solution, and this is important for many machine learning applications. For example, in the application of data hyper-cleaning {(see the experiment in Section \ref{sec: experiment} for more details)}, one aims to improve the learning performance by adding a regularizer to push the weights of the clean samples towards 1 while push the weights of the contaminated samples towards 0. Such a regularizer often takes a nonsmooth and nonconvex form.



To simplify the notation, throughout the paper we define the function $\Phi(x):= f(x, y^*(x))$.
We also adopt the following standard assumptions regarding the regularized bi-level optimization problem (P). 
\begin{assum}\label{assum:geo}
The functions in the regularized bi-level optimization problem (P) satisfy:
\begin{enumerate}[leftmargin=*,topsep=0pt,itemsep=.5mm]
    \item Function $g(x,\cdot)$ is $\mu$-strongly convex for all $x$ and function $\Phi(x)=f(x,y^*(x))$ is nonconvex;
    \item Function $h$ is proper and lower-semicontinuous (possibly nonsmooth and nonconvex);
    \item Function $(\Phi+h)(x)$ is bounded below and has bounded sub-level sets.
\end{enumerate}
\end{assum}

In Assumption \ref{assum:geo}, the regularizer $h$ can be any nonsmooth and nonconvex function so long as it is closed. This covers most of the regularizers that we use in practice. 
In addition to Assumption \ref{assum:geo}, we also impose the following Lipschitz continuity and smoothness conditions on the objective functions, which are widely considered in the existing literature \citep{ghadimi2018approximation,ji2020convergence}. In the following assumption, we denote $z:=(x,y)$.

\begin{assum}\label{assum:lip}
The functions $f(z)$ and $g(z)$ in the bi-level problem (P) satisfy:
\begin{enumerate}[leftmargin=*,topsep=0pt,itemsep=.5mm]
\item Function $f(z)$ is $M$-Lipschitz. Gradients $\nabla f(z)$ and $\nabla g(z)$ are $L$-Lipschitz;
    \item Jacobian $\nabla_x\nabla_y g(z)$ and Hessian $\nabla_y^2 g(z)$ are $\tau$-Lipschitz and $\rho$-Lipschitz, respectively.
\end{enumerate}
\end{assum}

Assumptions \ref{assum:geo} and \ref{assum:lip} imply that the mapping $y^*(x)$ is $\kappa$-Lipschitz, where $\kappa=L/\mu>1$ denotes the problem condition number \citep{lin2019gradient,chen2021proximal}.

Lastly, note that the problem (P) is rewritten as the regularized minimization problem $\min_{x\in \mathbb{R}^d} \Phi(x)+h(x)$, which can be nonsmooth and nonconvex. Therefore, our optimization goal is to {find a critical point $x^*$ of the function $\Phi(x)+h(x)$} that satisfies the optimality condition $\zero \in \partial (\Phi + h)(x^*)$. Here, $\partial$ denotes the following generalized notion of subdifferential.

\begin{definition}(Subdifferential and critical point, \citep{rockafellar2009variational})\label{def:sub}
	The Frech\'et subdifferential $\widehat\partial F$ of a function $F$ at $x\in \dom F$ is the set of $u\in \mathbb{R}^d$ defined as
	\begin{align*}
	\widehat\partial F(x) = \Big\{u: \liminf_{z\neq x, z\to x} \frac{F(z) - F(x) - u^\top (z-x)}{\|z-x\|} \ge 0 \Big\},
	\end{align*}
	and the limiting subdifferential $\partial F$ at $x\in \textnormal{dom}~F$ is the graphical closure of $\widehat\partial F$ defined as:
	\begin{align*}
	\partial F(x) \!=\! \big\{ \!u\!: \exists (x_k, F(x_k)) \!\to\! (x, F(x)), \widehat{\partial} F(x_k) \!\ni\! u_k  \!\to\! u \big\}.
	\end{align*}
	The set of {critical points} of $F$ is defined as $\{ x: \zero\in\partial F(x) \}$. 
\end{definition}

\section{Proximal Bi-level Optimization with AID}
In this section, we introduce the proximal bi-level optimization algorithm with momentum accelerated approximate implicit differentiation (referred to as {proximal BiO-AIDm}).
Recall that $\Phi(x):= f(x, y^*(x))$. The main challenge for solving the regularized bi-level optimization problem (P) is the computation of the gradient $\nabla \Phi(x)$, which involves higher-order derivatives of the lower-level function. Fortunately, this gradient can be effectively estimated using the popular AID scheme as we elaborate below. 

First, note that $\nabla \Phi(x)$ has been shown in \citep{ji2020bilevel} to take the following analytical form.
\begin{align}
\nabla &\Phi(x_k) =
\nabla_x f(x_k,y^*(x_k)) -\nabla_x \nabla_y g(x_k,y^*(x_k)) v_k^*, \nonumber 
\end{align}
where $v_k^*$ corresponds to the solution of the linear system $\nabla_y^2 g(x_k,y^*(x_k))v=
\nabla_y f(x_k,y^*(x_k))$. In particular, $y^*(x_k)$ is the minimizer of the strongly convex function $g(x_k,\cdot)$, and it can be effectively approximated by running $T$ {Nesterov's} accelerated gradient descent updates on $g(x_k,\cdot)$ and obtaining the output $y_{k+1}$ as the approximation. With this approximated minimizer, the AID scheme estimates the gradient $\nabla \Phi(x_k)$ as follows:
\begin{align}\label{hyper-aid}
\text{(AID):}\quad \widehat\nabla &\Phi(x_k) = \nonumber\\
&\nabla_x f(x_k,y_{k+1}) -\nabla_x \nabla_y g(x_k,y_{k+1})\widehat{v}_k^*,
\end{align}
where $\widehat{v}_k^*$ is the solution of the approximated linear system $\nabla_y^2 g(x_k,y_{k+1}) v =
\nabla_y f(x_k,y_{k+1})$, which can be efficiently solved by standard conjugate-gradient (CG) solvers.  
For simplicity of the discussion, we assume that $\widehat{v}_k^*$ is exactly computed throughout the paper. Moreover, the Jacobian-vector product involved in \cref{hyper-aid} can be efficiently computed using the existing automatic differentiation packages \citep{domke2012generic,grazzi2020iteration}.

Based on the estimated gradient $\widehat{\nabla} \Phi(x)$, we can then apply the standard proximal gradient algorithm (a.k.a. forward-backward splitting) \citep{lions1979splitting} to solve the regularized optimization problem (P). This algorithm is referred to as proximal BiO-AIDm and is summarized in \Cref{alg:main_deter}. Specifically, in each outer loop $k$, we first run $T$ accelerated gradient descent steps {with Nesterov's momentum with initial point $y_k$} to minimize  $g(x_k,\cdot)$ and find an approximated minimizer $y_{k+1}=y^T(x_k,y_k)\approx y^*(x_k)$, where we use the notation $y^T(x_k,y_k)$ to emphasize the dependence on $x_k$ and $y_k$. Then, this approximated minimizer
is utilized by the AID scheme to estimate $\nabla \Phi(x_k)$. Finally, we apply the proximal gradient algorithm to minimize the regularized objective function $\Phi(x)+h(x)$. Here, the proximal mapping of any function $h$ at $v$ is defined as
\begin{align*}
    \prox{h} (v):= \argmin_{u\in\mathbb{R}^d} \Big\{h(u) + \frac{1}{2}\|u-v\|^2\Big\}.
\end{align*} 

\begin{algorithm}[tbh]
	\red{\caption{Proximal bi-level optimization with momentum accelerated AID (proximal BiO-AIDm)}}
	\small
	\label{alg:main_deter}
	\begin{algorithmic}[1]
		\STATE {\bfseries Input:}  Stepsizes $\alpha, \beta >0$, {momentum parameter $\eta>0$,} initializations $x_0, y_0$.
		\FOR{$k=0,1,2,...,{K-1}$}
		\STATE{Set $y^0(x_k,y_k)=u_{0}=y_k$}
		\FOR{$t=1,....,T$}
		\vspace{0.05cm}
		\STATE{{$u_{t} = y^{t-1}(x_k,y_k)-\alpha \nabla_y g(x_k,y^{t-1}(x_k,y_k))$}}
		\STATE{{$y^t(x_k,y_k) = u_{t} + \eta\big(u_{t}-u_{t-1}\big)$}}
		\vspace{0.05cm}
		\ENDFOR
        \STATE{Set $y_{k+1}=y^{T}(x_{k},y_{k})$}
        \STATE{AID: compute $\widehat\nabla \Phi(x_k)$ according to \cref{hyper-aid} }
        \STATE{Update $x_{k+1} \in \prox{\beta h} (x_k- \beta \widehat\nabla \Phi(x_k))$}
		\ENDFOR
		{\STATE {\bfseries Output:} $x_K$, $y_K$.}
	\end{algorithmic}
\end{algorithm}

Under Assumption \ref{assum:geo}.1 and Assumption \ref{assum:lip}, the following lemma characterizes the smoothness of $\Phi$ and the gradient estimation error $\|\widehat \nabla \Phi(x_k)- \nabla \Phi(x_k)\|$ of the AID scheme. Throughout the paper, we denote $\kappa=\frac{L}{\mu}$ as the condition number of the bi-level problem (P). 

\begin{lemma}[\cite{ghadimi2018approximation}]\label{le:aidhy}
Let Assumptions~\ref{assum:geo}.1 and \ref{assum:lip} hold. Then, function $\Phi$ is differentiable and the gradient $\nabla\Phi$ is $L_\Phi$-Lipschitz with $L_{\Phi} = L + \frac{2L^2+\tau M^2}{\mu} + \frac{\rho L M+L^3+\tau M L}{\mu^2} + \frac{\rho L^2 M}{\mu^3}=\mathcal{O}(\kappa^3)$.
Moreover, the gradient estimate obtained by the AID scheme satisfies
\begin{align}
\|\widehat \nabla \Phi(x_k)- \nabla \Phi(x_k)\|^2 \leq \Gamma \|y_{k+1} - y^*(x_k)\|^2. \nonumber
\end{align}
where $\Gamma =3L^2+\frac{3\tau^2 M^2}{\mu^2} + 6L^2\big(1+\sqrt{\kappa}\big)^2\big(\kappa +\frac{\rho M}{\mu^2}\big)^2=\mathcal{O}(\kappa^3).$ 
\end{lemma}

\section{Global Convergence and Complexity of Proximal BiO-AID}
In this section, we study the global convergence properties of proximal BiO-AIDm for general regularized nonconvex and nonsmooth bi-level optimization. 

First, note that the main update of proximal BiO-AIDm in \Cref{alg:main_deter} follows from the proximal gradient algorithm, which has been proven to generate a convergent optimization trajectory to a critical point in general nonconvex optimization \citep{Attouch2009}. Hence, one may expect that proximal BiO-AIDm should share the same convergence guarantee. However, this is not obvious as the proof of convergence of the proximal gradient algorithm heavily relies on the fact that it is a descent-type algorithm, i.e., the objective function is strictly decreasing over the iterations. As a comparison, the main update of proximal BiO-AIDm applies an approximated gradient $\widehat\nabla \Phi(x_k)$, which is correlated with both the upper- and lower-level objective functions through the AID scheme and destroys the descent property of the proximal gradient update, and hence conceals the proof of convergence. This is the key challenge for proving the formal convergence of proximal BiO-AIDm in general nonsmooth and nonconvex optimization.

Our key result next proves that proximal BiO-AIDm does admit an intrinsic potential function that is monotonically decreasing over the iterations. Therefore, it is indeed a descent-type algorithm, which is the first step toward establishing the global convergence. 

\begin{restatable}{proposition}{propositionlyapunov}\label{lemma: lyapunov}
	Let Assumptions~\ref{assum:geo} and~\ref{assum:lip} hold and define the potential function
	\begin{align}
		H(x,y'):= \Phi(x) +h(x) +\frac{7}{8}\|y^T(x,y')-y^*(x)\|^2. \label{eq: lyapunov}
	\end{align}
	Choose hyperparameters $\alpha=\frac{1}{L}$, $\beta \le \frac{1}{2}(L_{\Phi}+\Gamma+\kappa^2)^{-1}=\mathcal{O}(\kappa^{-3})$, $\eta=\frac{\sqrt{\kappa}-1}{\sqrt{\kappa}+1}$ and $T\ge\mathcal{O}(\sqrt\kappa\ln\kappa)$. Then, the parameter sequence $\{x_k\}_k$ generated by \Cref{alg:main_deter} satisfies, for all $k=1,2,...,$ 
\begin{align}
H(x_{k+1}&, y_{k+1}) \le H(x_k,y_k)- \frac{1}{4\beta} \|x_{k+1}-x_k\|^2 \nonumber\\
&- \frac{1}{8} \Big(\|y_{k+1}-y^*(x_k)\|^2 + \|y_{k+2}-y^*(x_{k+1})\|^2 \Big). \nonumber
\end{align}
\end{restatable}
To elaborate, the potential function $H$ consists of two components: the upper-level objective function $\Phi(x)+h(x)$ and a regularization term $\|y^T(x,y')-y^*(x)\|^2$ that tracks the optimality gap of the lower-level optimization. Hence, the potential function $H$ fully characterizes the optimization goal of the entire bi-level optimization. Intuitively, if $\{x_k\}_k$ converges to a certain critical point $x^*$ and $y^T(x_k,y_k)$ converges to $y^*(x^*)$, then it can be seen that $H(x_k,y_k)$ will converge to the local optimum $(\Phi+h)(x^*)$. 

Based on the above characterization of the potential function, we obtain the following global convergence result of proximal BiO-AIDm in general regularized nonconvex optimization.
\begin{restatable}{theorem}{theorema}\label{thm: 1}
	Under the same conditions as those of \Cref{lemma: lyapunov}, the parameter sequence $\{x_k, y_k\}_k$ generated by \Cref{alg:main_deter} satisfies the following properties.
	\begin{enumerate}[leftmargin=*,topsep=0pt,itemsep=.5mm]
		\item $\|x_{k+1} - x_k \| \overset{k}{\to} 0$, 
		$\|y_{k+1} - y^*(x_k)\| \overset{k}{\to} 0$;
		\item The function value sequence $\{(\Phi+h)(x_k)\}_k$ converges to a finite limit $H^*>-\infty$;
		\item The sequence $\{(x_k,y_k)\}_k$ is bounded and has a compact set of limit points. Moreover, $(\Phi+h)(x^*)\equiv H^*$ for any limit point $x^*$ of $\{x_k\}_k$;
		\item Every limit point $x^*$ of $\{x_k\}_k$ is a critical point of the upper-level function $(\Phi+h)(x)$.
	\end{enumerate}
\end{restatable}
\Cref{thm: 1} provides a comprehensive characterization of the global convergence properties of proximal BiO-AIDm in regularized nonconvex and nonsmooth bi-level optimization. Specifically, item 1 shows that the parameter sequence $\{x_k\}_k$ is asymptotically stable, and $y_{k+1}$ asymptotically converges to the corresponding minimizer $y^*(x_k)$ of the lower-level objective function $g(x_k, \cdot)$. In particular, in the unregularized case (i.e., $h=0$), this result reduces to the existing understanding that the gradient norm $\|\nabla \Phi(x)\|$ converges to zero \citep{ji2020bilevel,ghadimi2018approximation,hong2020two}, which does not imply the convergence of the parameter sequence. Item 2 shows that the function value sequence converges to a finite limit, which is also the limit of the potential function value sequence $\{H(x_k,y_k)\}_k$. Moreover, items 3 and 4 show that the parameter sequence $\{x_k\}_k$ converges to only critical points of the objective function, and these limit points are in a flat region where the corresponding function value is the constant $H^*$. To summarize, \Cref{thm: 1} formally proves that proximal BiO-AIDm will eventually converge to critical points in nonsmooth and nonconvex bi-level optimization. 

{In addition to the above global convergence result, \Cref{lemma: lyapunov} can be further leveraged to characterize the computation complexity of proximal BiO-AIDm for finding a critical point in regularized nonconvex bi-level optimization. Specifically, when the regularizer $h$ is convex, we can define the following proximal gradient mapping associated with the objective function $\Phi(x)+h(x)$. 
\begin{align}
    G(x)=\frac{1}{\beta}\Big(x-\prox{\beta h} \big(x- \beta \nabla \Phi(x)\big)\Big). \label{Phi_pg}
\end{align}
The proximal gradient mapping is a standard metric for evaluating the optimality of regularized nonconvex optimization problems \citep{Nesterov2014}. It can be shown that $x$ is a critical point of $\Phi(x)+h(x)$ if and only if $G(x)=\zero$, and it reduces to the normal gradient in the unregularized case. Hence, we define the \textbf{convergence criterion} as finding a near-critical point $x$ that satisfies $\|G(x)\|\le \epsilon$ for some pre-determined accuracy $\epsilon>0$. We obtain the following global convergence rate and complexity of proximal BiO-AIDm.

\begin{restatable}{corollary}{coroalg}\label{coro_alg1}
Suppose $h$ is convex and the conditions of \Cref{lemma: lyapunov} hold. Then, the sequence $\{x_k\}_k$ generated by \Cref{alg:main_deter} satisfies the following convergence rate.
\begin{align}
    \min_{0\le k\le K}\|G(x_k)\| \le \sqrt{\frac{32}{K\beta}\big(H(x_0) - \inf_x (\Phi+g)(x)\big)}. \label{eq: globalrate}
\end{align}
Moreover, to achieve $\min_{0\le k\le K}\|G(x_k)\|\le\epsilon$, {we run the algorithm with $K=\mathcal{O}(\beta^{-1}\epsilon^{-2})=\mathcal{O}(\kappa^{3}\epsilon^{-2})$ outer iterations and $T=\mathcal{O}(\sqrt{\kappa}\ln\kappa)$ inner iterations, and} the overall computation complexity is $KT=\mathcal{O}(\kappa^{3.5}(\ln\kappa)\epsilon^{-2})$. 
\end{restatable}
With the momentum accelerated AID scheme, proximal BiO-AIDm achieves a computation complexity $\mathcal{O}(\kappa^{3.5}\epsilon^{-2})$ in regularized nonsmooth and nonconvex bi-level optimization, which strictly improves the computation complexity $\mathcal{O}(\kappa^4\epsilon^{-2})$ of BiO-AID that only applies to smooth nonconvex bi-level optimization \citep{ji2020bilevel}. 
To the best of our knowledge, this is the first convergence rate and complexity result of momentum accelerated algorithm for solving regularized nonsmooth and nonconvex bi-level optimization problems. We note that another momentum accelerated bi-level optimization algorithm has been studied in \citep{ji2021lower}, which only applies to unregularized (strongly) convex bi-level optimization problems.

}

\section{Convergence Rates under Local Nonconvex Geometry}
In the previous section, we have proved that the optimization trajectory generated by proximal BiO-AIDm approaches a compact set of critical points. Hence, we are further motivated to exploit the local function geometry around the critical points to study its local convergence guarantees, which is the focus of this section. In particular, we consider a broad class of {\L}ojasiewicz-type geometry of nonconvex functions. 

\subsection{Local Kurdyka-{\L}ojasiewicz Geometry}
General nonconvex functions typically do not have a global geometry. However, they may have certain local geometry around the critical points that determines the local convergence rate of optimization algorithms. In particular, the Kurdyka-{\L}ojasiewicz (K{\L}) geometry characterizes a broad spectrum of local geometries of nonconvex functions \citep{Bolte2007,Bolte2014}, and it generalizes various conventional global geometries such as the strong convexity and Polyak-{\L}ojasiewicz geometry. Next, we formally introduce the K{\L} geometry.

\begin{definition}[\KL geometry, \cite{Bolte2014}]\label{def: KL}
	A proper and lower semi-continuous function $F$ is said to have the \KL geometry if for every compact set $\Omega\subset \mathrm{dom}F$ on which $F$ takes a constant value $F_\Omega \in \mathbb{R}$, there exist $\varepsilon, \lambda >0$ such that for all $\bar{x} \in \Omega$ and all $x\in \{z\in \mathbb{R}^m : \dist_\Omega(z)<\varepsilon, F_\Omega < F(z) <F_\Omega + \lambda\}$, the following condition holds:
	\begin{align}\label{eq: KL}
	\varphi' \left(F(x) - F_\Omega\right) \cdot \dist_{\partial F(x)}(\zero) \ge 1,
	\end{align}
	where $\varphi'$ is the derivative of $\varphi: [0,\lambda) \to \mathbb{R}_+$ that takes the form $\varphi(t) = \frac{c}{\theta} t^\theta$ for certain constant $c>0$ and \KL parameter $\theta\in (0,1]$, and $\dist_{\partial F(x)}(\zero) = \min_{u \in \partial F} \|u-\zero\|$ denotes the point-to-set distance.
\end{definition}

As an intuitive explanation, when function $F$ is differentiable, the \KL inequality in \cref{eq: KL} can be rewritten as $F(x)-F_{\Omega} \le \mathcal{O}(\|\nabla F(x)\|^{\frac{1}{1-\theta}})$, which can be viewed as a type of local gradient dominance condition and generalizes the Polyak-\L ojasiewicz (PL) condition (with parameter $\theta=\frac{1}{2}$) \citep{Lojasiewicz1963, Karimi2016}. 
In the existing literature, a large class of functions has been shown to have the local \KL geometry, e.g., sub-analytic functions, logarithm and exponential functions and semi-algebraic functions \citep{Bolte2014}.   
Moreover, the \KL geometry has been exploited to establish the convergence of many gradient-based algorithms in nonconvex optimization, e.g., gradient descent \citep{Attouch2009,Li2017}, accelerated gradient method \citep{Zhou-ijcai2020}, alternating minimization \citep{Bolte2014} and distributed gradient methods \citep{Zhou2016}.

\subsection{Convergence Rates of Proximal BiO-AIDm under \KL Geometry}
In this subsection, we obtain the following asymptotic {function value} convergence rates of the proximal BiO-AIDm algorithm under different parameter ranges of the \KL geometry. Throughout, 
we define $k_0\in\mathbb{N}^+$ to be a sufficiently large integer. {We also adopt the following mild assumption that $y^*(\cdot)$ is a sub-differentiable mapping.}

\begin{assum}\label{assum: H}
	Function $\|y^T(x,y)-y^*(x)\|^2$ is sub-differentiable, i.e., $\partial_{(x,y)}(\|y^T(x,y)-y^*(x)\|^2)\ne \emptyset$.
\end{assum}

\begin{restatable}{theorem}{theoremobjrate}\label{thm:objrate}
    {Let Assumptions~\ref{assum:geo}, \ref{assum:lip} and \ref{assum: H} hold and and assume that the potential function $H$ defined in \cref{eq: lyapunov} has \KL geometry. Then, under the same choices of hyper-parameters as those of \Cref{lemma: lyapunov}}, the potential function value sequence $\{H(x_k,y_k)\}_k$ converges to its limit $H^*$ at the following rates. 
    \begin{enumerate}[leftmargin=*,topsep=0pt,itemsep=.5mm]
        \item If K\L~geometry holds with $\theta\in\big(\frac{1}{2},1\big)$, then $H(x_k,y_k)\downarrow H^*$ {super-linearly} as
        \begin{align}
            H(x_k,y_k)-H^*\le {\mathcal{O}}\Big(\!-\!\Big(\frac{1}{2(1-\theta)}\Big)^{k-k_0}\Big), ~\forall k\ge k_0; \label{eq: superlinear_converge} 
        \end{align}
        \item If K\L~geometry holds with $\theta=\frac{1}{2}$, then $H(x_k,y_k)\downarrow H^*$ {linearly} as (for some constant $C>0$)
        \begin{align}
            H(x_k,y_k)-H^*\le {\mathcal{O}}\big((1+C)^{-(k-k_0)}\big),\quad \forall k\ge k_0;  \label{eq: linear_converge} 
        \end{align}
        \item If K\L~geometry holds with $\theta\in\big(0,\frac{1}{2}\big)$, then $H(x_k,y_k)\downarrow H^*$ sub-linearly as
        \begin{align}
            H(x_k,y_k)-H^*\le {\mathcal{O}}\big((k-k_0)^{-\frac{1}{1-2\theta}}\big),\quad \forall k\ge k_0, \label{eq: sublinear_converge} 
        \end{align}
    \end{enumerate}
\end{restatable}
Intuitively, a larger \KL parameter $\theta$ implies that the local geometry of the potential function $H$ is sharper, which implies an orderwise faster convergence rate as shown in \Cref{thm:objrate}. In particular, when the K\L~geometry holds with $\theta=\frac{1}{2}$, the proximal BiO-AIDm algorithm converges at a linear rate, which matches the convergence rate of bi-level optimization under the stronger geometry that both the upper and lower-level objective functions are strongly convex \citep{ghadimi2018approximation}. 
To the best of our knowledge, the above result provides the first {function value} converge rate analysis of proximal BiO-AIDm in the full spectrum of the nonconvex local \KL geometry.

\begin{figure*}
    \centering
     \begin{subfigure}[b]{0.235\textwidth}
        \centering 
        \includegraphics[width=1.13\textwidth]{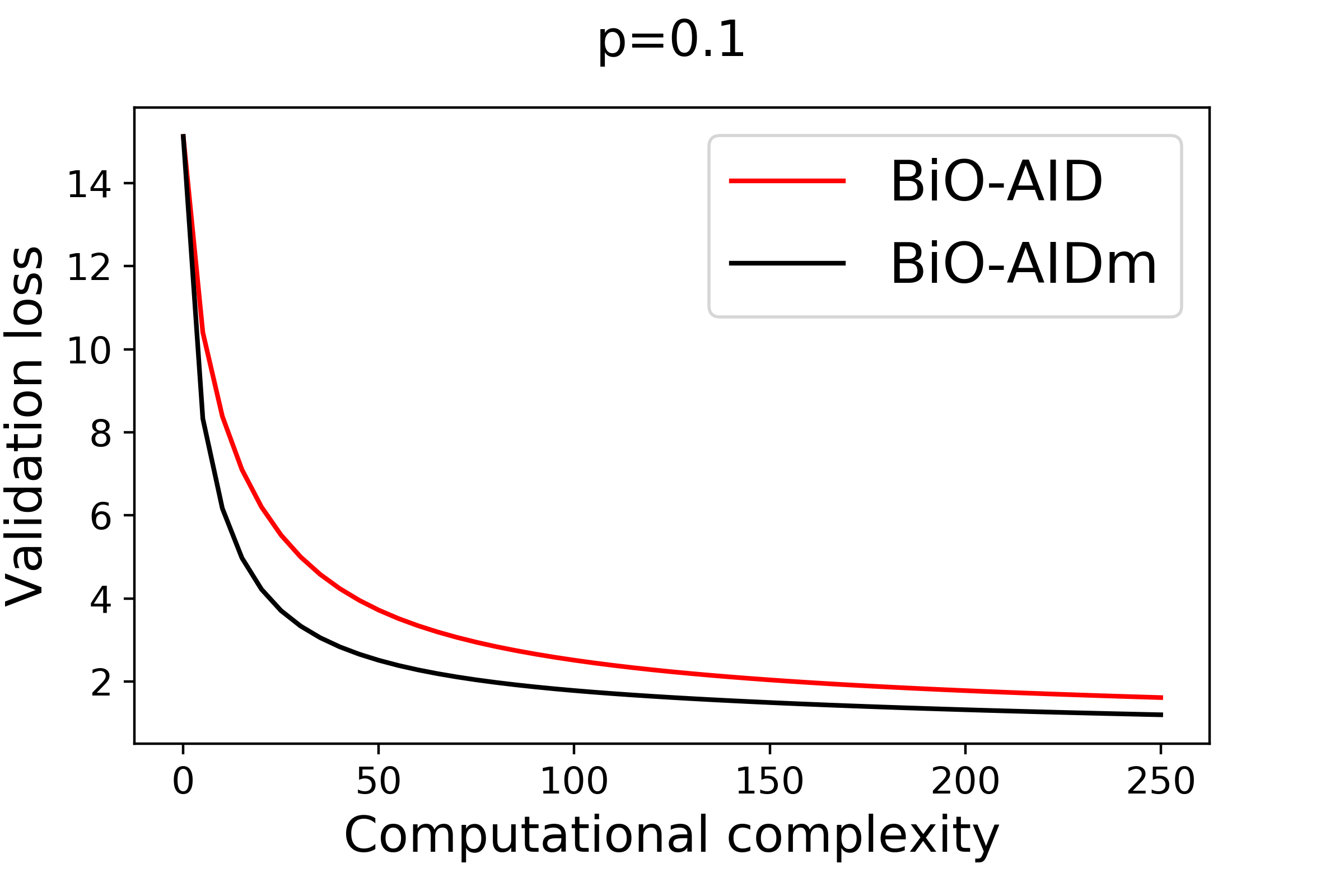}
    \end{subfigure}
    \begin{subfigure}[b]{0.235\textwidth}
        \centering 
        \includegraphics[width=1.13\textwidth]{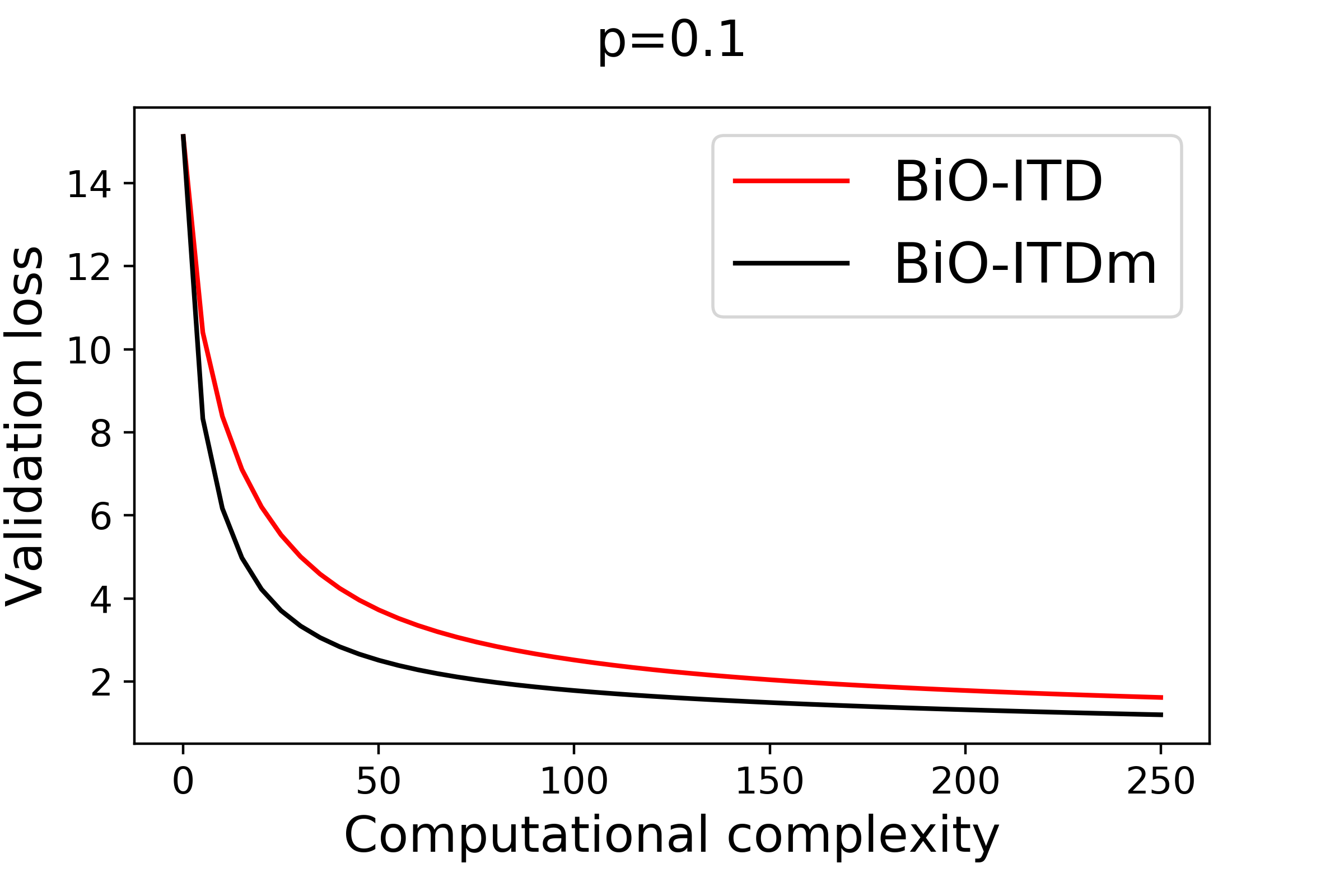}
    \end{subfigure}
    \begin{subfigure}[b]{0.235\textwidth}
        \centering 
        \includegraphics[width=1.13\textwidth]{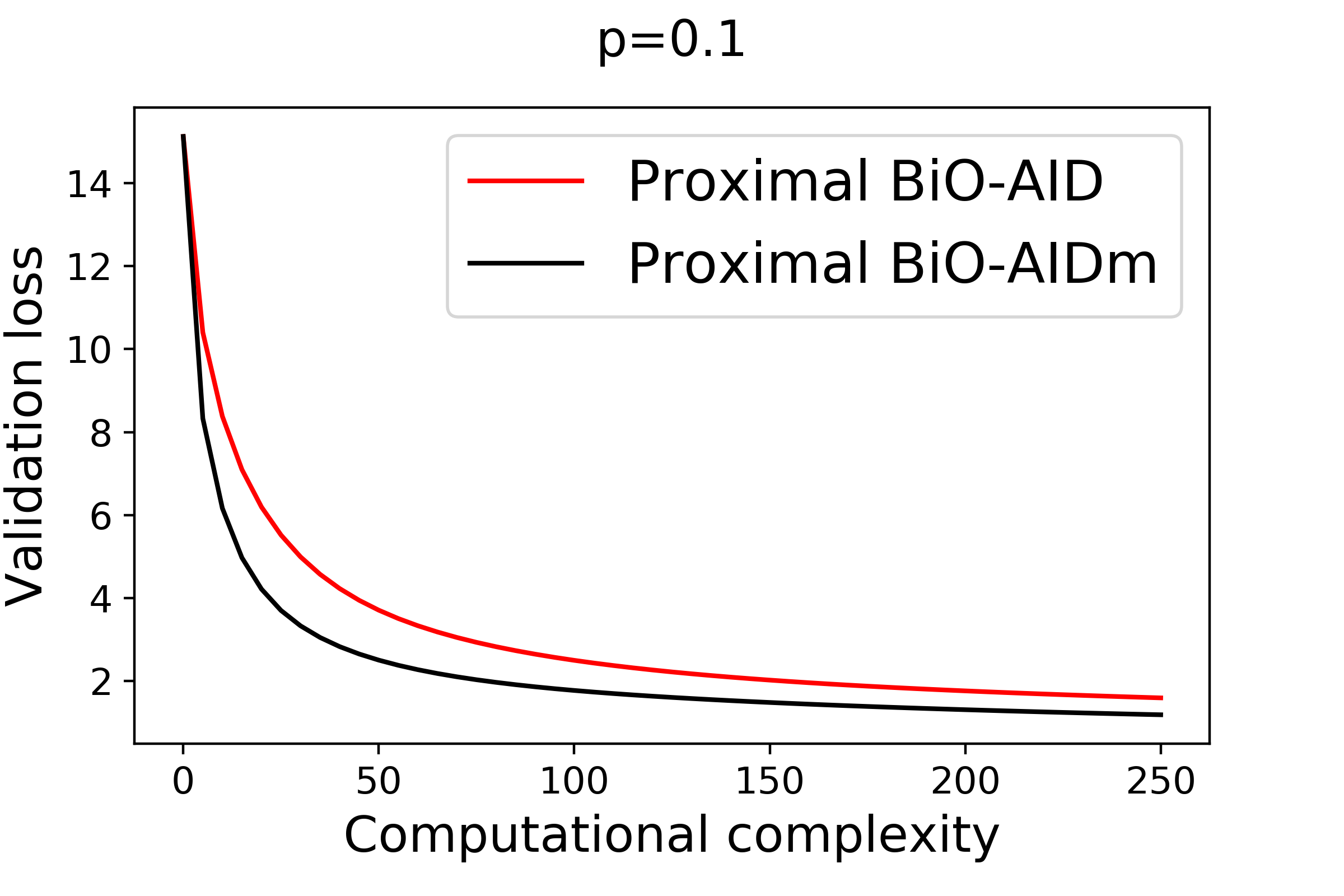}
    \end{subfigure}
    \begin{subfigure}[b]{0.235\textwidth}
        \centering 
        \includegraphics[width=1.13\textwidth]{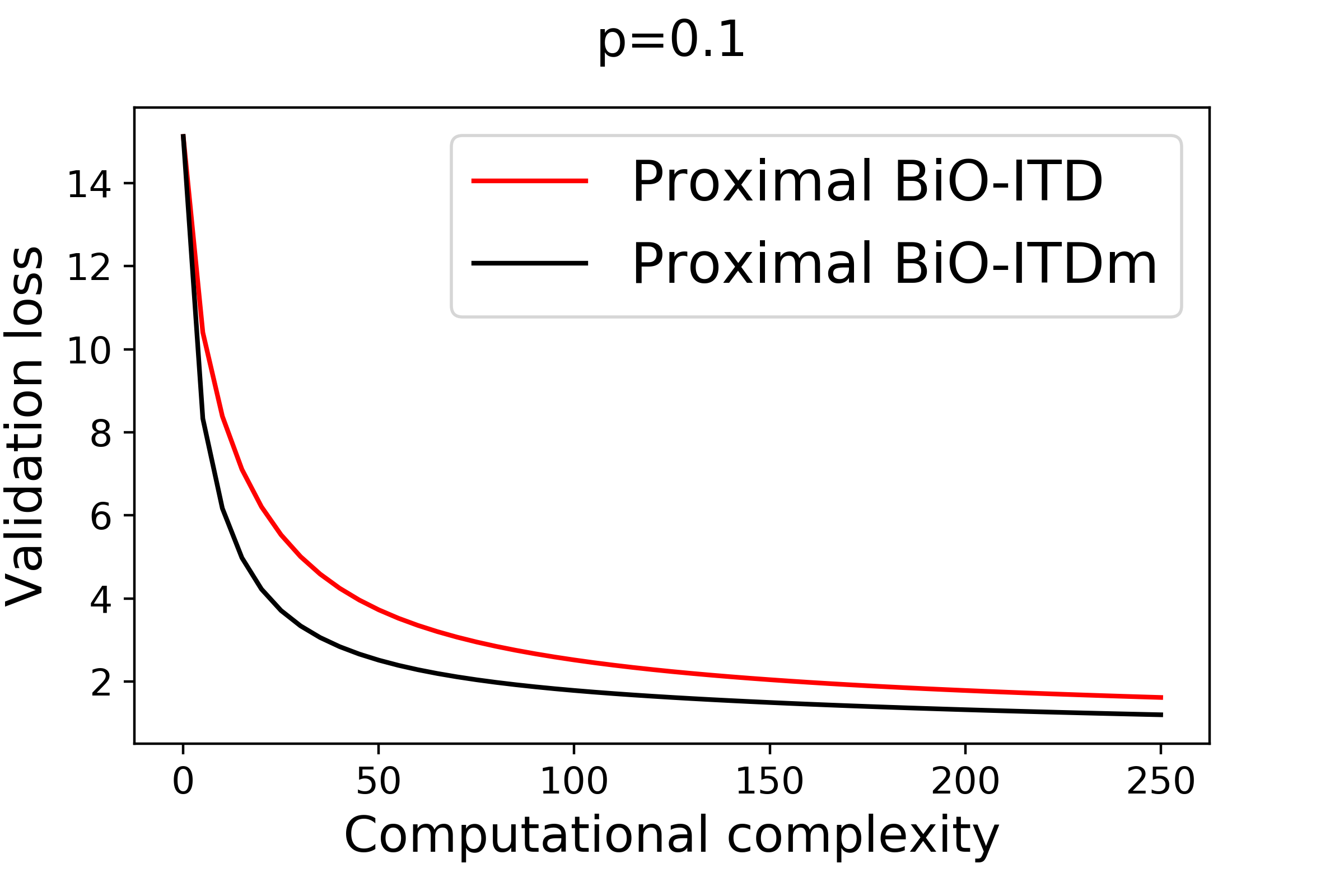}
    \end{subfigure}\\
    
     \begin{subfigure}[b]{0.235\textwidth}
        \centering 
        \includegraphics[width=1.13\textwidth]{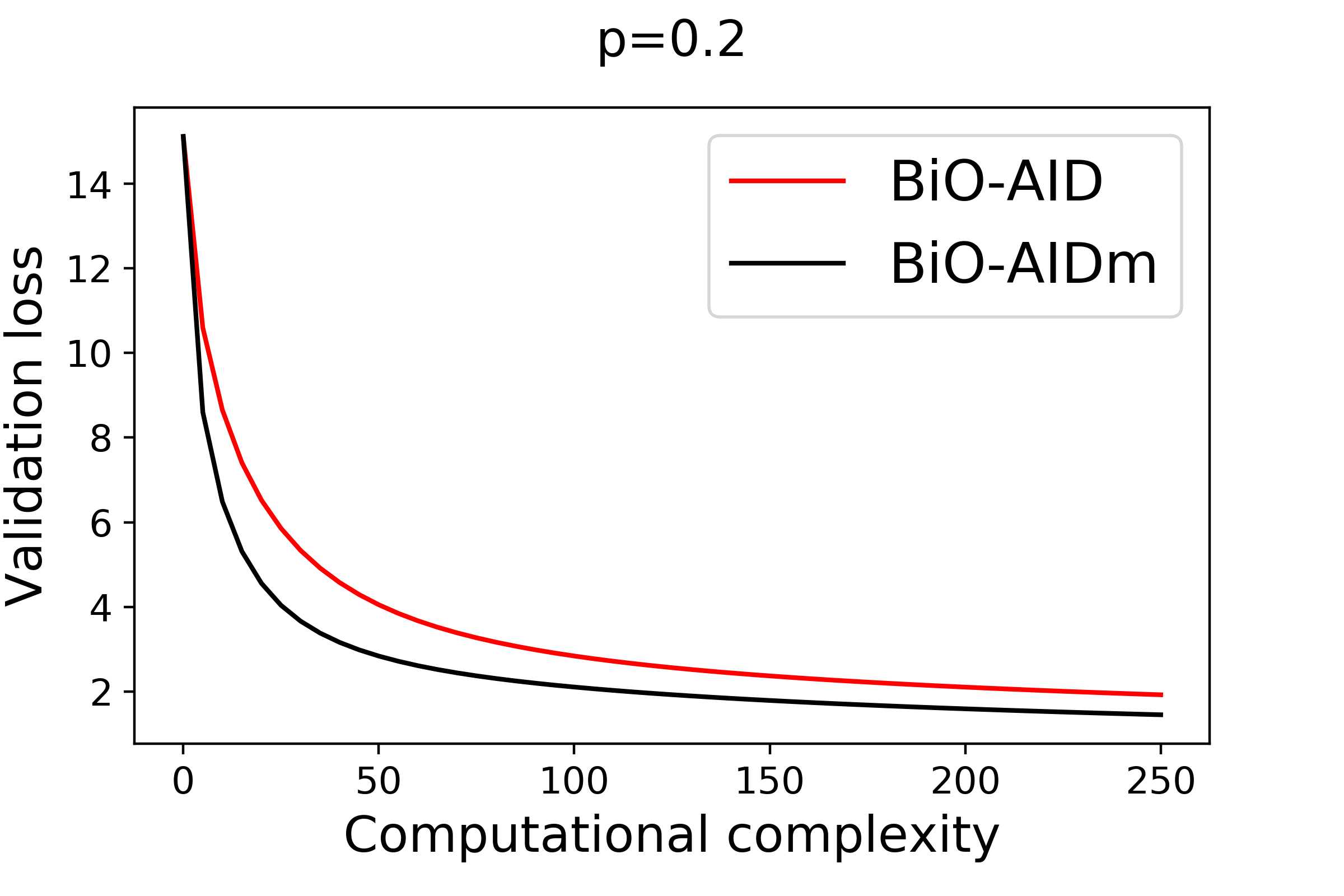}
    \end{subfigure}
    \begin{subfigure}[b]{0.235\textwidth}
        \centering 
        \includegraphics[width=1.13\textwidth]{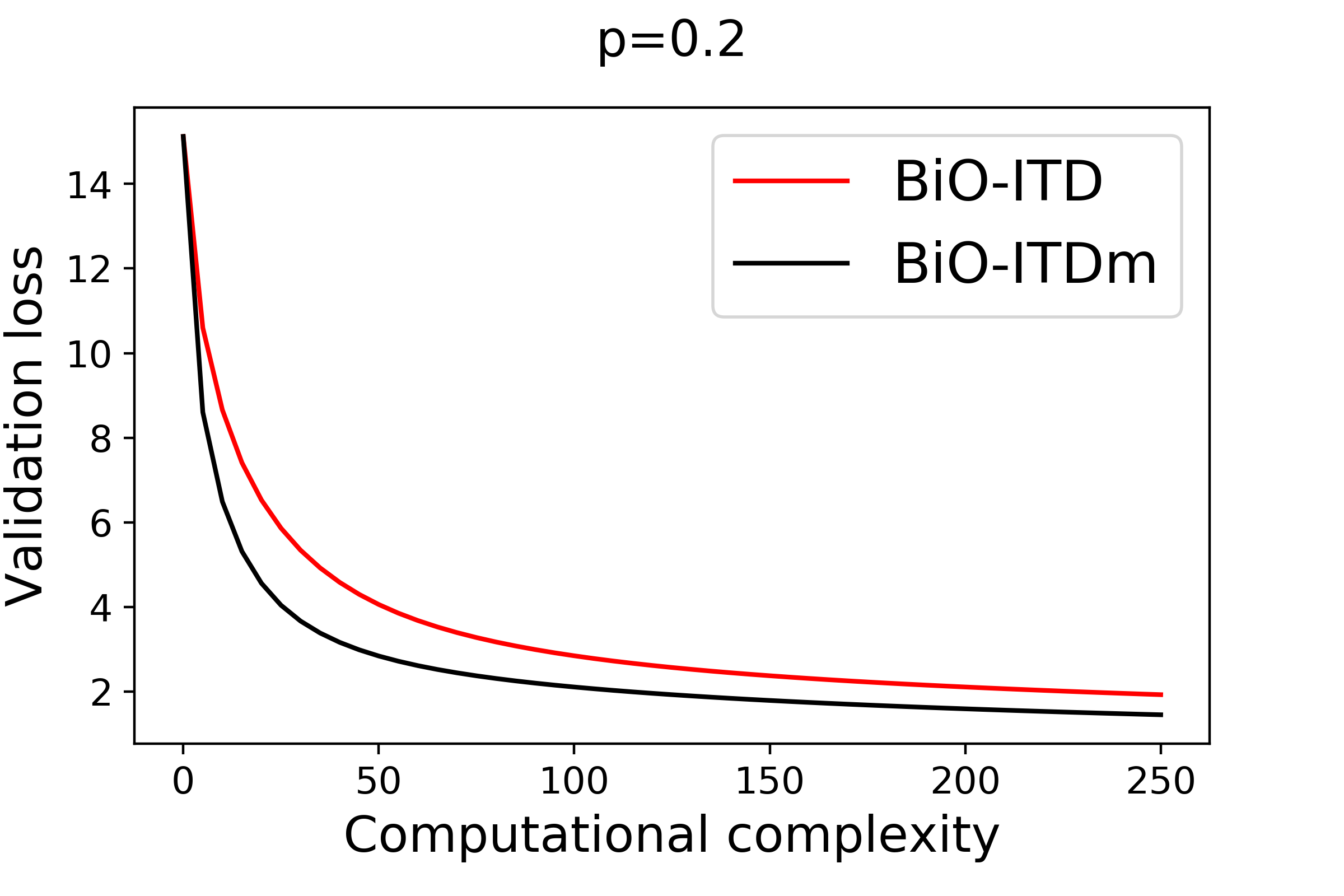}
    \end{subfigure}
    \begin{subfigure}[b]{0.235\textwidth}
        \centering 
        \includegraphics[width=1.13\textwidth]{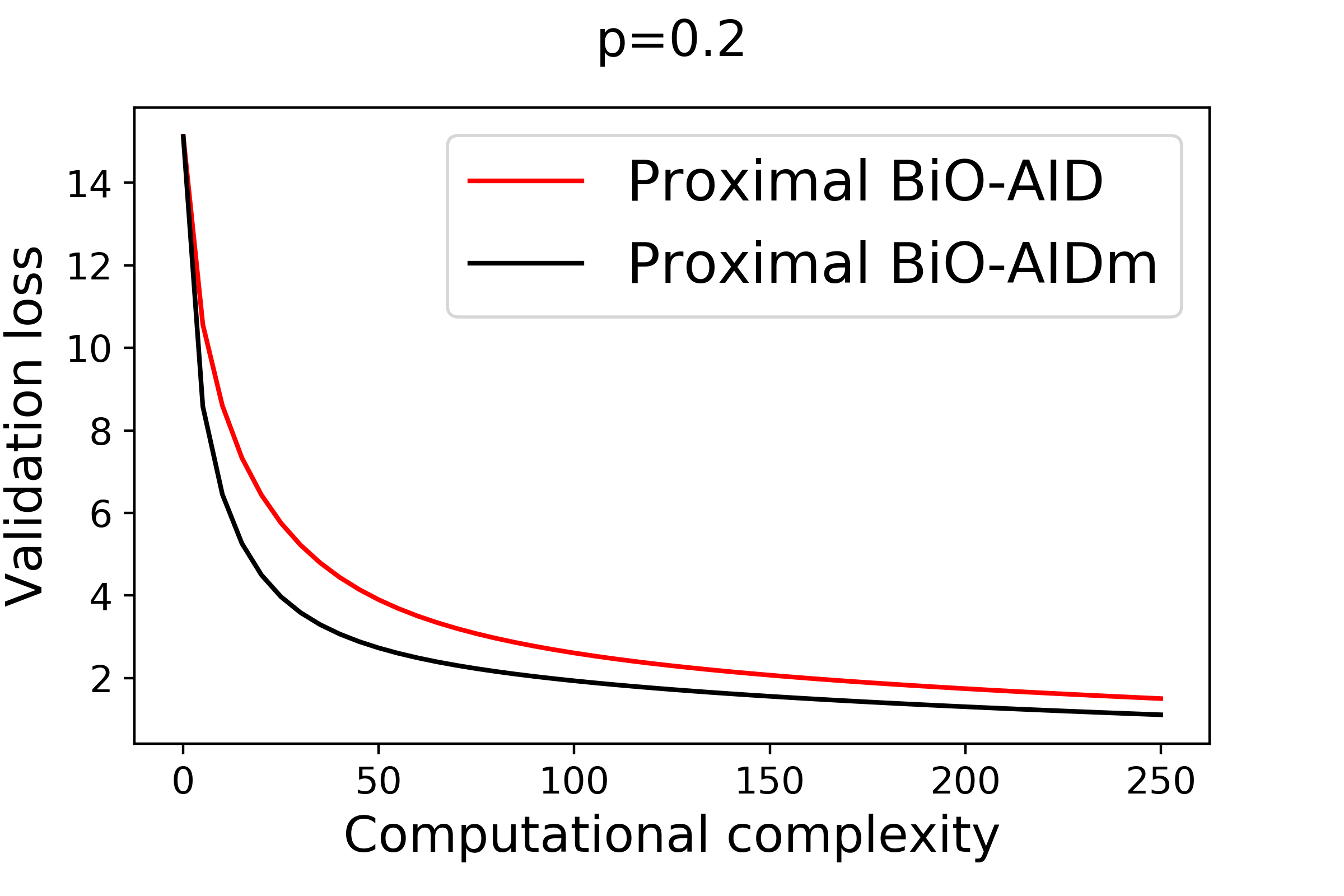}
    \end{subfigure}
    \begin{subfigure}[b]{0.235\textwidth}
        \centering 
        \includegraphics[width=1.13\textwidth]{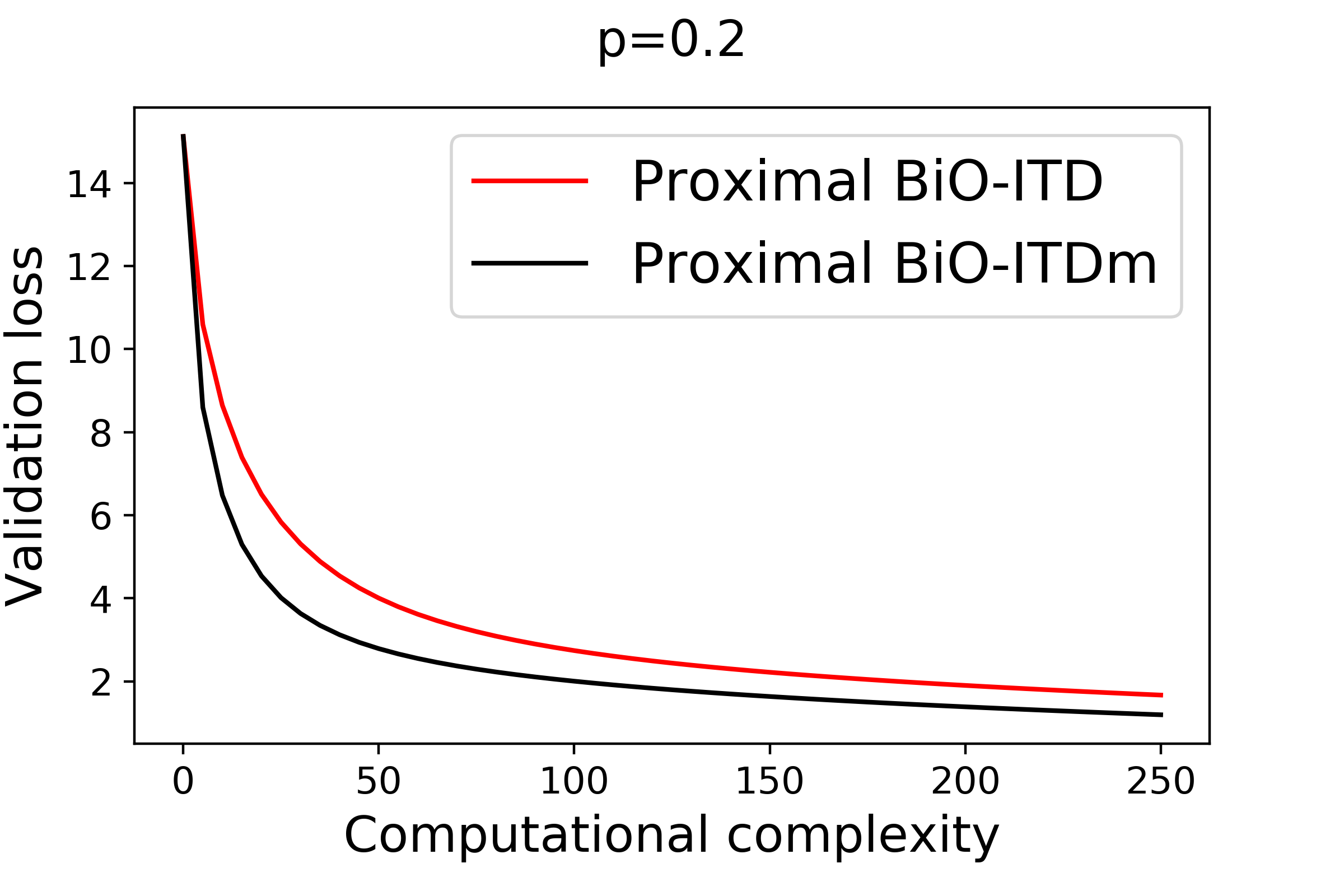}
    \end{subfigure}\\
    
    \begin{subfigure}[b]{0.235\textwidth}
        \centering 
        \includegraphics[width=1.13\textwidth]{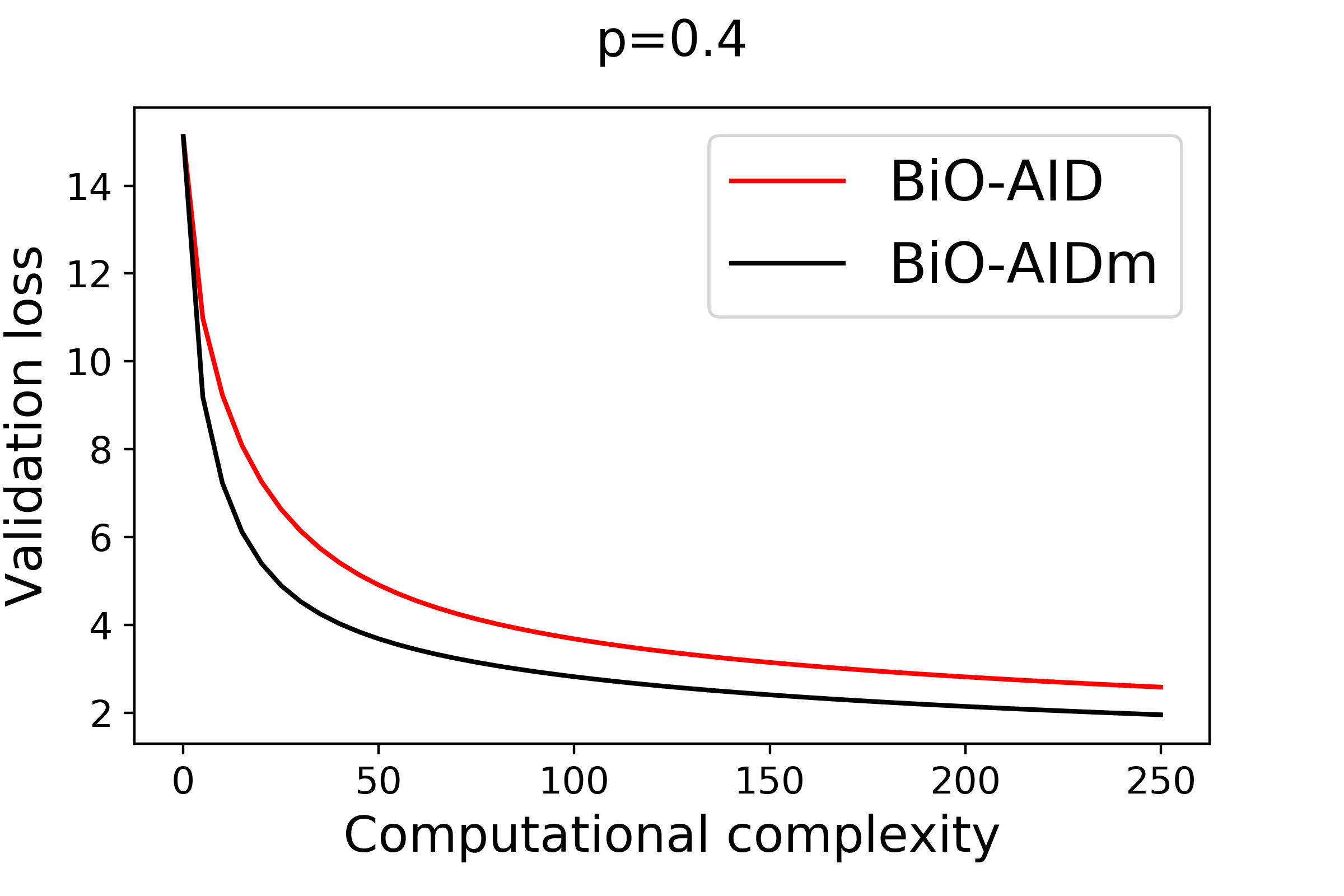}
    \end{subfigure}
    \begin{subfigure}[b]{0.235\textwidth}
        \centering 
        \includegraphics[width=1.13\textwidth]{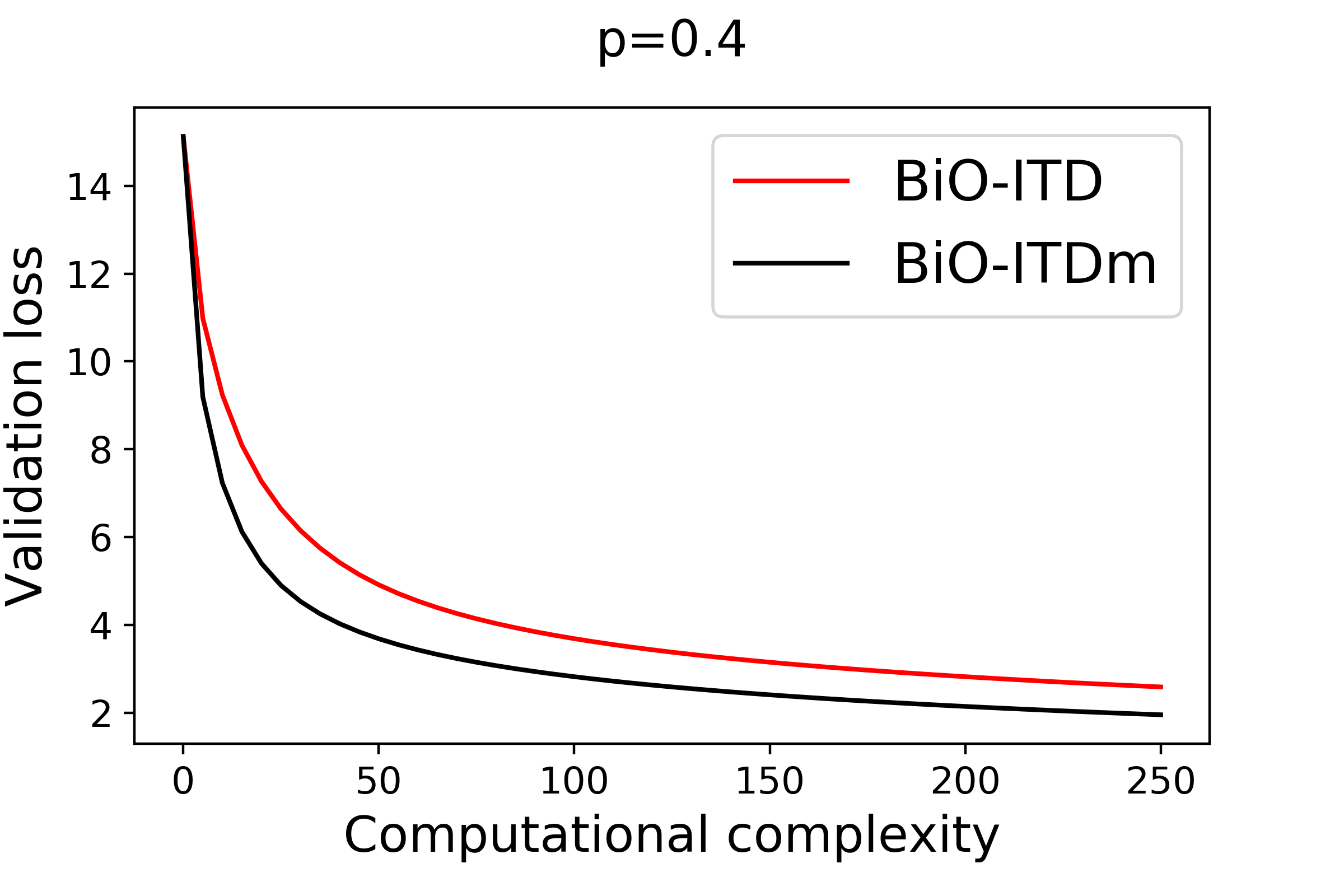}
    \end{subfigure}
    \begin{subfigure}[b]{0.235\textwidth}
        \centering 
        \includegraphics[width=1.13\textwidth]{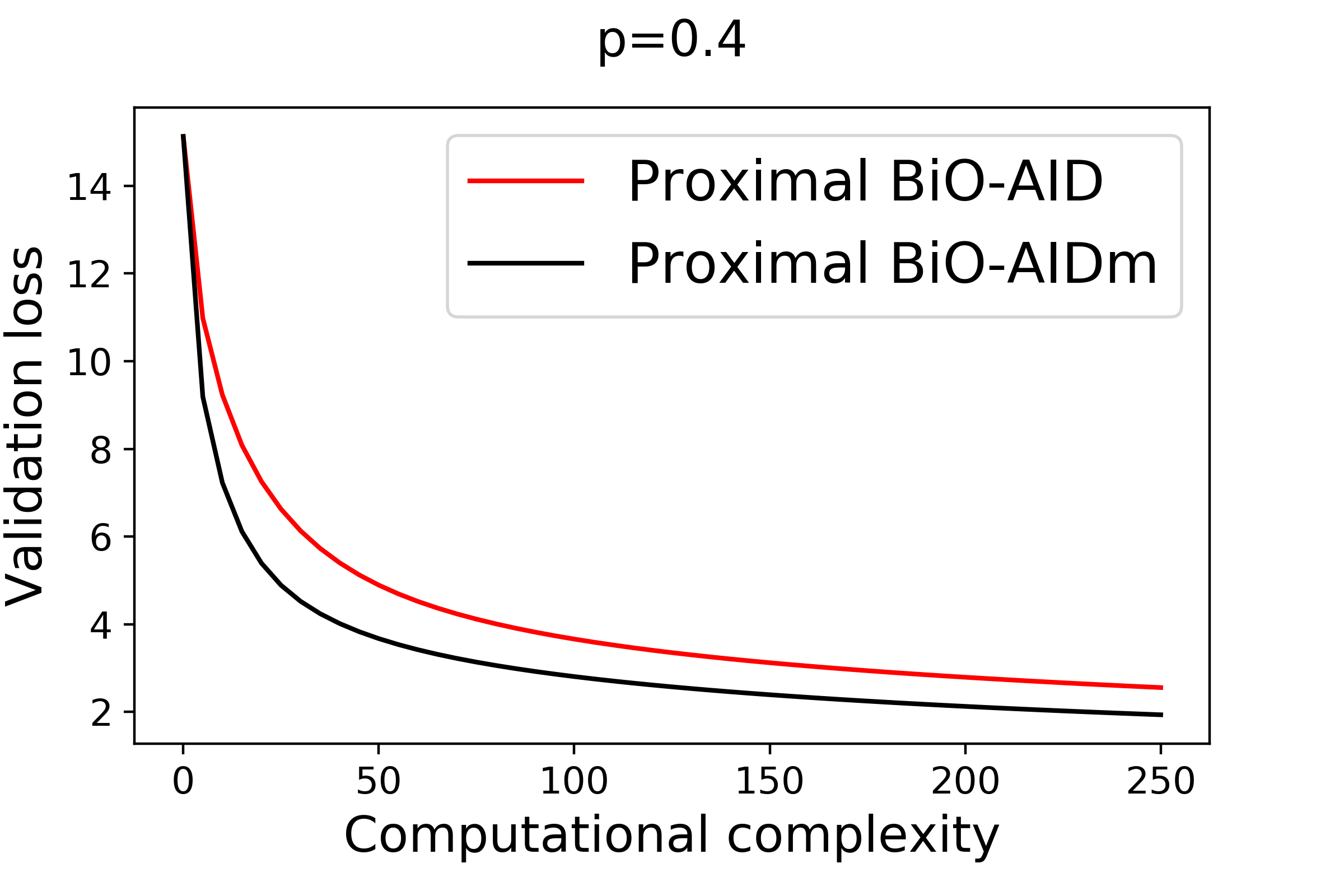}
    \end{subfigure}
    \begin{subfigure}[b]{0.235\textwidth}
        \centering 
        \includegraphics[width=1.13\textwidth]{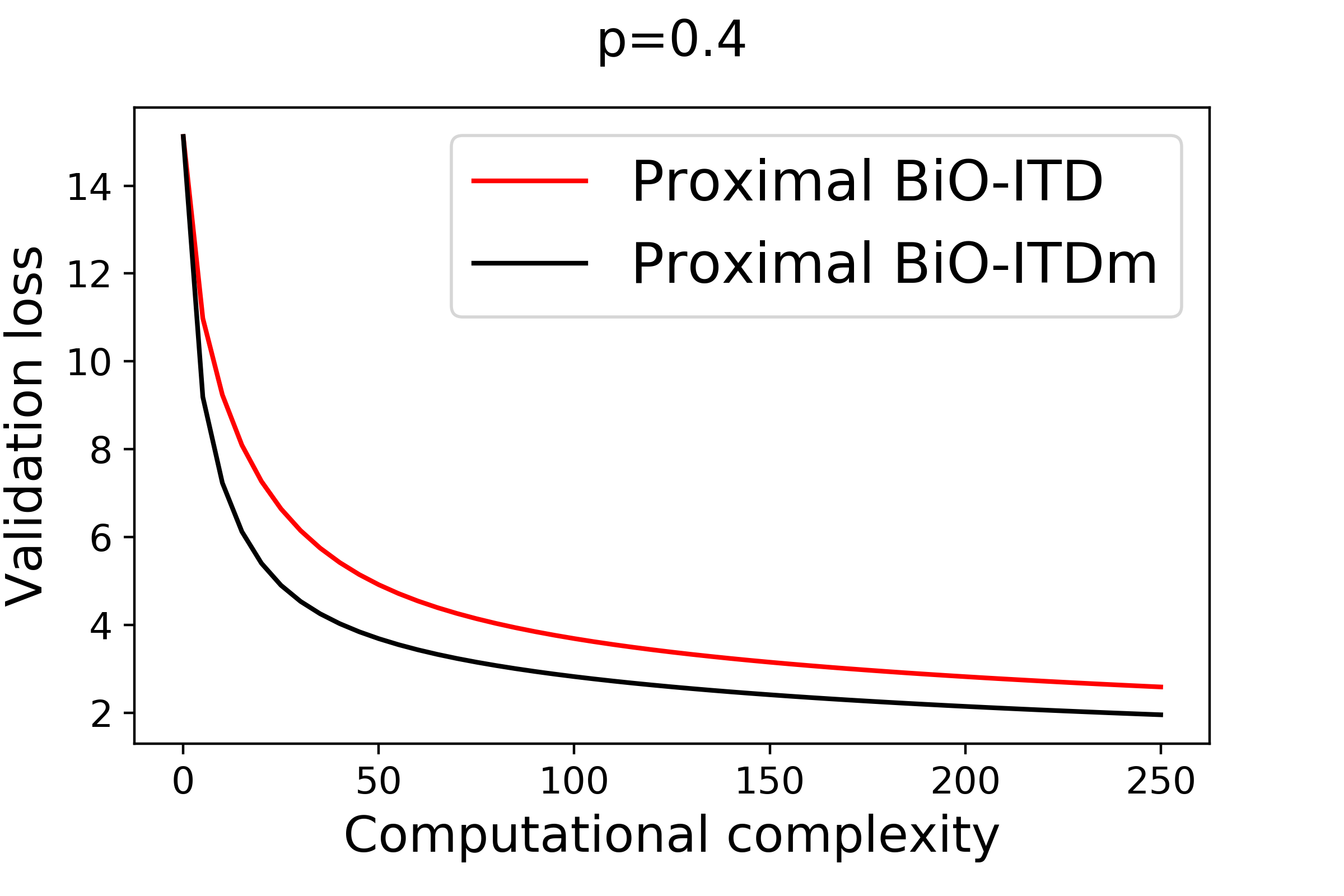}
    \end{subfigure}
    \vspace{-2mm}
	\caption{Comparison of bi-level optimization algorithms under data corruption rate $p=0.1$ (top row), $p=0.2$ (middle row) and $p=0.4$ (bottom row). The $y$-axis corresponds to the 
	upper-level objective function value, and the $x$-axis corresponds to the overall computation complexity (number of inner gradient descent steps).}\label{fig: outer} 
	\vspace{0mm}
\end{figure*}

\section{Experiment}\label{sec: experiment}
We apply our bi-level optimization algorithm to solve a regularized data cleaning problem \cite{shaban2019truncated} with the MNIST dataset \cite{lecun1998gradient} and a linear classification model. We generate a training dataset $\mathcal{D}_{\text{tr}}$ with 20k samples, a validation dataset $\mathcal{D}_{\text{val}}$ with 5k samples, and a test dataset with 10k samples. In particular, we corrupt the training data by randomizing a proportion $p\in (0,1)$ of their labels, and the goal of this application is to identify and avoid using these corrupted training samples. The corresponding bi-level problem is written as follows. 
\begin{align}
    &\min_{\lambda} \frac{1}{|\mathcal{D}_{\text{val}}|}\!\sum_{(x_i,y_i)\in \mathcal{D}_{\text{val}}} \!\!\!\Big(L\big(w^*(\lambda)^{\top}x_i,y_i\big)\nonumber - \gamma\min(|\lambda_i|,a)\Big), \nonumber\\
    &{\rm where}~w^*(\lambda)=\frac{1}{|\mathcal{D}_{\text{tr}}|}\!\sum_{(x_i,y_i)\in\mathcal{D}_{\text{tr}}}\!\!\!\!\!\!\! \sigma(\lambda_i)L\big(w^{\top}\!x_i,y_i\big)+\rho\|w\|^2, \nonumber
\end{align}
where $x_i, y_i$ denote the data and label of the $i$-th sample, respectively, $\sigma(\cdot)$ denotes the sigmoid function, $L$ is the cross-entropy loss, and $\rho,\gamma>0$ are regularization hyperparameters. The regularizer $\rho\|w\|^2$ makes the lower-level objective function strongly convex. In particular, we add the nonconvex and nonsmooth regularizer $-\gamma\min(|\lambda_i|,a)$ to the upper-level objective function. Intuitively, it encourages $|\lambda_i|$ to approach the large positive constant $a$ so that the training sample coefficient $\sigma(\lambda_i)$ is close to either 0 or 1 for corrupted and clean training samples, respectively. 
In this experiment we set $a=20$. Therefore, such a regularized bi-level data cleaning problem belongs to the problem class considered in this paper. 

We compare the performance of our proximal BiO-AIDm with several bi-level optimization algorithms, including proximal BiO-AID (without accelerated AID), BiO-AID (without accelerated AID) and BiO-AIDm (with accelerated-AID). In particular, for BiO-AID and BiO-AIDm, we apply them to solve the unregularized data cleaning problem (i.e., $\gamma=0$). This serves as a baseline that helps understand the impact of regularization on the test performance. In addition, we also implement all these algorithms by replacing the AID scheme with the ITD scheme to demonstrate their generality. 


\textbf{Hyperparameter setup.} We consider choices of corruption rates $p=0.1, 0.2, 0.4$, regularization parameters $\gamma=0.5,5.0$ and $\rho=10^{-3}$. We run each algorithm for $K=50$ outer iterations with stepsize $\beta=0.5$ and $T=5$ inner gradient steps with stepsize $\alpha=0.1$. For the algorithms with momentum accelerated AID/ITD, we set the momentum parameter $\eta=1.0$. 

\begin{table*}
\caption{Comparison of test accuracy (test loss) under different corruption rates $p$. }\label{table}
\centering
\begin{tabular}{llccc}
\hline
\multicolumn{2}{c}{}                  & $p=0.1$ & $p=0.2$ & $p=0.4$ \\ \hline
\multicolumn{2}{l}{BiO-AID} & 71.00\% (1.4933)  &  61.40\% (1.7959)  &  47.50\% (2.4652) \\ \hline
\multicolumn{2}{l}{BiO-AIDm}   & 73.80\% (1.1090) &  63.80\% (1.3520)  &  49.20\% (1.8713) \\ \hline
\multirow{2}{*}{proximal BiO-AID} & $\gamma=0.5$  &   71.00\% (1.4931)   &  61.40\% (1.7957)   &     47.50\% (2.4649)    \\ \cline{2-5} 
                               & $\gamma=5.0$  &   71.00\% (1.4914)    &    61.50\% (1.7938)   &     47.50\% (2.4628)    \\ \hline
\multirow{2}{*}{proximal BiO-AIDm} & $\gamma=0.5$  &   73.80\% (1.1089)   &   63.80\% (1.3519)     &    49.20\% (1.8711) \\ \cline{2-5} 
                               & $\gamma=5.0$  &    73.90\% (1.1081)   &   63.80\% (1.3511)     &   49.20\% (1.8702)     \\ \hline\hline
\multicolumn{2}{l}{BiO-ITD}    & 71.10\% (1.4971)  &  61.50\% (1.8000) &    47.50\% (2.4714)  \\ \hline
\multicolumn{2}{l}{BiO-ITDm}  &  73.80\% (1.1102) &    63.80\% (1.3534)  &   49.10\% (1.8738)  \\ \hline
\multirow{2}{*}{proximal BiO-ITD}  & $\gamma=0.5$  &   71.10\% (1.4968)    &   61.50\% (1.7996)     &    47.50\% (2.4708)     \\ \cline{2-5} 
                               & $\gamma=5.0$  &  71.20\% (1.4947)    &   61.60\% (1.7960)     &    47.50\% (2.4652)     \\ \hline
\multirow{2}{*}{proximal BiO-ITDm}  & $\gamma=0.5$  &  73.80\% (1.1100) &  63.80\% (1.3530)   & 49.10\% (1.8732)  \\ \cline{2-5} 
                               & $\gamma=5.0$  &  73.90\% (1.1079) &  64.10\% (1.3495) & 49.40\% (1.8680) \\ \hline
\end{tabular}
\end{table*}

\subsection{Optimization Performance}
We first investigate the effect of momentum acceleration on the optimization performance. In Figure \ref{fig: outer}, we plot the upper-level objective function value versus the computational complexity for different bi-level algorithms under different data corruption rates. In these figures, we separately compare the non-proximal algorithms and the proximal algorithms, as their upper-level objective functions are different (non-proximal algorithms are applied to solve the unregularized bi-level problem). 
It can be seen that all the bi-level optimization algorithms with momentum accelerated AID/ITD schemes consistently converge faster than their unaccelerated counterparts. The reason is that the momentum scheme accelerates the convergence of the inner gradient descent steps, which yields a more accurate implicit gradient and thus accelerates the convergence of the outer iterations. 

\subsection{Test Performance}

To understand the impact of {momentum and} the nonconvex regularization on the test performance of the model, we report the test accuracy and test loss of the models trained by all the algorithms in Table \ref{table}. It can be seen that the bi-level optimization algorithms with momentum accelerated AID/ITD achieve significantly better test performance than their unaccelerated counterparts. This demonstrates the advantage of introducing momentum to accelerate the AID/ITD schemes. Also, our proximal BiO-AIDm and proximal BiO-ITDm achieve the best test performance among all the cases. Furthermore, we observe that the test loss decreases as the regularizer coefficient $\gamma$ increases. Therefore, adding such a regularizer improves test performance via distinguishing the sample coefficients $\sigma(\lambda_i)$ between corrupted and clean training samples. Lastly, a larger corruption rate $p$ leads to a lower test performance, which is reasonable. 


\section{Conclusion} 
In this paper, we provided a comprehensive analysis of the proximal BiO-AIDm algorithm with momentum acceleration for solving regularized nonconvex and nonsmooth bi-level optimization problems. Our key finding is that this algorithm admits an intrinsic monotonically decreasing potential function, which fully tracks the bi-level optimization progress. Based on this result, we established the first global convergence rate of proximal BiO-AIDm to a critical point in {regularized nonconvex optimization}, which is faster than that of BiO-AID. We also characterized the asymptotic convergence behavior and rates of the algorithm under the local \KL geometry. We anticipate that this new analysis framework can be extended to study the convergence of other bi-level optimization algorithms, including stochastic bi-level optimization. In particular, it would be interesting to explore how bi-level optimization algorithm design affects the form of the potential function and leads to different convergence guarantees and rates in nonconvex bi-level optimization. 


{
	\bibliographystyle{icml2022}
	\bibliography{./ref}
}

\newpage
\onecolumn
\appendix

\addcontentsline{toc}{section}{Appendix} 
\part{Supplementary Materials} 
\parttoc 
\allowdisplaybreaks

\section{Proof of \Cref{lemma: lyapunov}}
\propositionlyapunov*

\begin{proof}
    
Based on the smoothness of the function $\Phi(x)$ established in~\Cref{le:aidhy}, we have 
\begin{align}
\Phi(x_{k+1}) &\le \Phi(x_k)  + \langle \nabla \Phi(x_k), x_{k+1}-x_k\rangle + \frac{L_\Phi}{2} \|x_{k+1}-x_k\|^2. \label{eq: 2}
\end{align}
On the other hand, by the definition of the proximal gradient step of $x_k$, we have 
\begin{align}
	h(x_{k+1}) + \frac{1}{2\beta} \|x_{k+1} - x_k + \beta \widehat{\nabla} \Phi(x_k)\|^2 \le h(x_{k}) + \frac{1}{2\beta} \|\beta \widehat{\nabla} \Phi(x_k)\|^2,
\end{align}
which further simplifies to 
\begin{align}
h(x_{k+1}) \le h(x_{k}) - \frac{1}{2\beta} \|x_{k+1} - x_k \|^2  - \inner{x_{k+1} - x_k}{\widehat{\nabla} \Phi(x_k)}. \label{eq: 1}
\end{align}
Adding up \cref{eq: 1} and \cref{eq: 2} yields that
\begin{align}
&\Phi(x_{k+1}) + h(x_{k+1})\nonumber\\
&\le \Phi(x_k) +h(x_k) -  \Big(\frac{1}{2\beta} - \frac{L_\Phi}{2} \Big) \|x_{k+1}-x_k\|^2 + \langle x_{k+1}-x_k, \nabla \Phi(x_k) - \widehat{\nabla} \Phi(x_k)\rangle  \nonumber\\
&\le \Phi(x_k)+h(x_k)  - \Big(\frac{1}{2\beta} -\frac{L_\Phi}{2} \Big) \|x_{k+1}-x_k\|^2  + \| x_{k+1}-x_k\| \| \nabla \Phi(x_k) - \widehat{\nabla} \Phi(x_k)\| \nonumber\\
&\le \Phi(x_k) +h(x_k) - \Big(\frac{1}{2\beta} -\frac{L_\Phi}{2} - \frac{\Gamma}{2}\Big) \|x_{k+1}-x_k\|^2  + \frac{1}{2\Gamma}\| \nabla \Phi(x_k) - \widehat{\nabla} \Phi(x_k)\|^2. \nonumber\\
&\le \Phi(x_k) +h(x_k) - \Big(\frac{1}{2\beta} -\frac{L_\Phi}{2} - \frac{\Gamma}{2}\Big) \|x_{k+1}-x_k\|^2  + \frac{1}{2}\|y_{k+1}-y^*(x_k)\|^2, \label{eq: 4}
\end{align}
where the last inequality utilizes \Cref{le:aidhy}. Next, note that {$y_{k+2}=y^T(x_{k+1},y_{k+1})$ is generated by minimizing the strongly-convex function $g$ through $T$ gradient descent steps with Nesterov's momentum with the initial point $y_{k+1}$}. Hence, with $\alpha = \frac{1}{L}$ and $\eta=\frac{\sqrt{\kappa}-1}{\sqrt{\kappa}+1}$ {\cite{Nesterov2014}}, we obtain that
\begin{align}
	&\|y_{k+2}-y^*(x_{k+1})\|^2 \nonumber\\
	&\le (1+\kappa)(1-\kappa^{-0.5})^T \|y_{k+1}-y^*(x_{k+1})\|^2 \nonumber\\
	&\le (1+\kappa)(1-\kappa^{-0.5})^T \big(2\|y_{k+1}-y^*(x_k)\|^2+2\|y^*(x_k)-y^*(x_{k+1})\|^2\big) \nonumber\\
	&\stackrel{(i)}{\le} \frac{1}{4} \|y_{k+1}-y^*(x_k)\|^2 + \frac{\kappa^2}{4}\|x_{k+1}-x_k\|^2,\label{eq: 3}
\end{align}
where (i) uses the fact that $y^*$ is $\kappa$-Lipschitz (proved in Proposition 1 of \cite{chen2021proximal}) and $T\ge \frac{\ln(8(1+\kappa))}{\ln((1-\kappa^{-0.5})^{-1})}=\mathcal{O}(\sqrt{\kappa}\ln\kappa)$.

Adding up \cref{eq: 3} and \cref{eq: 4} yields that
\begin{align}
&\Phi(x_{k+1}) + h(x_{k+1}) + \|y_{k+2}-y^*(x_{k+1})\|^2 \nonumber\\
&\le \Phi(x_k) +h(x_k) - \Big(\frac{1}{2\beta} -\frac{L_\Phi}{2} - \frac{\Gamma}{2} - \frac{\kappa^2}{4}\Big) \|x_{k+1}-x_k\|^2 + \frac{3}{4}\|y_{k+1}-y^*(x_k)\|^2 \nonumber\\
&\stackrel{(i)}{\le} \Phi(x_k) +h(x_k) - \Big(\frac{1}{2\beta} -\frac{L_\Phi}{2} - \frac{\Gamma}{2} - \frac{\kappa^2}{4}\Big) \|x_{k+1}-x_k\|^2 + \frac{3}{4}\|y_{k+1}-y^*(x_k)\|^2 \nonumber\\
&\stackrel{(ii)}{\le} \Phi(x_k) +h(x_k) - \frac{1}{4\beta} \|x_{k+1}-x_k\|^2 + \frac{3}{4}\|y_{k+1}-y^*(x_k)\|^2 \nonumber
\end{align}
{where (i) uses the number of iterations that $T\ge \frac{\ln 8}{\ln((1-\kappa^{-0.5})^{-1})}=\mathcal{O}(\sqrt{\kappa})$ to ensure that $(1-\kappa^{-1})^T\le \frac{1}{8}$, and (ii) uses the stepsize $\beta\le \frac{1}{2}(L_{\Phi}+\Gamma+\kappa^2)^{-1}$. Defining the potential function $H(x_k,y_k):= \Phi(x_k) +h(x_k) +\frac{7}{8}\|y^T(x_k,y_k)-y^*(x_k)\|^2=\Phi(x_k) +h(x_k) +\frac{7}{8}\|y_{k+1}-y^*(x_k)\|^2$ and rearranging the above inequality yields that}
\begin{align}
H(x_{k+1},y_{k+1}) &\le H(x_k,y_k) - \frac{1}{4\beta} \|x_{k+1}-x_k\|^2 -  \frac{1}{8} \big(\|y_{k+1}-y^*(x_k)\|^2 + \|y_{k+2}-y^*(x_{k+1})\|^2 \big). \nonumber
\end{align}
\end{proof}

\section{Proof of \Cref{thm: 1}}
\theorema*

\begin{proof}
We first prove the item 1. Summing \Cref{lemma: lyapunov} from $k=0,1,...,K-1$, we obtain that for all $K\in\mathbb{N}_+$,
\begin{align}
&\sum_{k=0}^{K-1} \frac{1}{4\beta} \|x_{k+1}-x_k\|^2 +  \frac{1}{8} \big(\|y_{k+1}-y^*(x_k)\|^2 + \|y_{k+2}-y^*(x_{k+1})\|^2 \big) \nonumber\\
&\le H(x_0,y_0) - H(x_K,y_K) \nonumber\\
&\stackrel{(i)}{\le} H(x_0,y_0) - \inf_x (\Phi+g)(x) \nonumber\\
&< +\infty.
\end{align}
where (i) uses $H(x,y)\ge (\Phi+g)(x)$ and the item 3 of Assumption \ref{assum:geo} that $\Phi+g$ is lower bounded.

Letting $K\to \infty$, we further obtain that 
\begin{align}
\sum_{k=0}^{\infty} \frac{1}{4\beta} \|x_{k+1}-x_k\|^2 +  \frac{1}{8} \big(\|y_{k+1}-y^*(x_k)\|^2 + \|y_{k+2}-y^*(x_{k+1})\|^2 \big) {<} +\infty. \label{eq: sum_dH}
\end{align}
Hence, we conclude that $\|x_{k+1}-x_k\|\to 0, \|y_{k+1}-y^*(x_k)\|\to 0$, which proves the item 1.


Next, we prove the item 2. We have shown in \Cref{lemma: lyapunov} that $\{H(x_k,y_k)\}_k$ is monotonically decreasing. Since $H(x_k,y_k){\ge \Phi(x_k)+g(x_k) \ge\inf_{x'} \Phi(x')+g(x')}$, which is bounded below, we conclude that $\{H(x_k,y_k)\}_k$ has a finite limit $H^*>-\infty$, i.e., $\lim_{k\to \infty} (\Phi+h)(x_k) + \frac{7}{8} \|y_{k+1} - y^*(x_k) \|^2 = H^*$. 
Moreover, since we already showed that $\|y_{k+1} - y^*(x_k) \| \to 0$, we further conclude that $\lim_{k\to \infty} (\Phi+h)(x_k) = H^*$. 

Next, we prove the item 3. {$\{x_k\}_k$ is bounded since $\Phi(x_k)+g(x_k)\le H(x_k,y_k)\le H(x_0,y_0)$ and $\Phi+g$ has compact sub-level set. Note that
\begin{align}
    \|y_k\|&\le \|y_k-y^*(x_{k-1})\|+\|y^*(x_{k-1})-y^*(0)\|+\|y^*(0)\|\nonumber\\
    &\overset{(i)}{\le} \|y_k-y^*(x_{k-1})\|+\kappa\|x_{k-1}\|+\|y^*(0)\|,
\end{align}
where (i) uses the $\kappa$-Lipschitz continuity of $y^*$ (Proved in Proposition 1 of \cite{chen2021proximal}). Since $\|y_k-y^*(x_{k-1})\|\to 0$ and $\|x_{k-1}\|$ is bounded, the above inequality implies that $\{x_k,y_k\}_k$ is bounded and thus has compact set of limit points.} 

Next, we bound the subdifferential of the function. By the optimality condition of the proximal gradient update of $x_k$ and the summation rule of subdifferential, we obtain that
\begin{align}
	\zero \in \partial h(x_{k+1}) + \frac{1}{\beta} \big(x_{k+1} - x_k + \beta \widehat{\nabla} \Phi(x_k)\big). \nonumber
\end{align}
The above equation further implies that 
\begin{align}
 \frac{1}{\beta} \big(x_{k} - x_{k+1}\big)  +\nabla  \Phi(x_{{k+1}}) - \widehat{\nabla} \Phi(x_k) \in \partial (\Phi+h)(x_{k+1}) . \nonumber
\end{align}
Then, we obtain that 
\begin{align}
	\dist_{\partial (\Phi+h)(x_{k+1})}(\zero) &\le  \frac{1}{\beta} \|x_{{k+1}} - x_{k}\| +\|\nabla  \Phi(x_{k+1}) - \widehat{\nabla} \Phi(x_k)\| \nonumber\\
	&\le  \frac{1}{\beta} \|x_{k} - x_{k+1}\| +{\|\nabla  \Phi(x_{k+1}) - \nabla \Phi(x_k)\|} +\|\nabla  \Phi(x_k) - \widehat{\nabla} \Phi(x_k)\| \nonumber\\
	&\overset{(i)}{\le} \Big(\frac{1}{\beta}+{L_{\Phi}}\Big) \|x_{k} - x_{k+1}\| + \sqrt{\Gamma}\|y_{k+1} - y^*(x_k)\|, \label{eq: dPhih}
\end{align}
where (i) follows from \Cref{le:aidhy}. Since we have shown that $\|x_{k+1}-x_k\|\to 0, \|y_{k+1}-y^*(x_k)\|\to 0$, the above inequality implies that 
\begin{align}
&\frac{1}{\beta} \big(x_{k} - x_{k+1}\big)  +\nabla  \Phi(x_{{k+1}}) - \widehat{\nabla} \Phi(x_k) \in \partial (\Phi+h)(x_{k+1}), \nonumber\\
\text{and} \quad & \frac{1}{\beta} \big(x_{k} - x_{k+1}\big)  +\nabla \Phi(x_{{k+1}}) - \widehat{\nabla} \Phi(x_k) \to \zero. \label{eq: 5}
\end{align}
Next, consider any limit point $x^*$ of $\{x_k\}_k$ so that $x_{k(j)} \overset{j}{\to} x^*$ along a subsequence. By the proximal update of $x_{k(j)}$, we have 
\begin{align}
h(x_{k(j)}) + &\frac{1}{2\beta} \|x_{k(j)} -x_{k(j)-1} \|^2 + \inner{x_{k(j)} -x_{k(j)-1} }{\widehat{\nabla} \Phi(x_{k(j)-1})}  \nonumber\\ 
&\le h(x^*) + \frac{1}{2\beta} \|x^* -x_{k(j)-1} \|^2 + \inner{x^* -x_{k(j)-1} }{\widehat{\nabla} \Phi(x_{k(j)-1})}. \nonumber
\end{align}
Rearranging the above inequality yields that
\begin{align}
&h(x_{k(j)}) + \frac{1}{2\beta} \|x_{k(j)} -x_{k(j)-1} \|^2 \nonumber\\
&\le h(x^*) + \frac{1}{2\beta} \|x^* -x_{k(j)-1} \|^2 + \inner{x^* -x_{k(j)} }{\widehat{\nabla} \Phi(x_{k(j)-1})- \nabla \Phi(x_{k(j)-1})+\nabla \Phi(x_{k(j)-1})} \nonumber\\
&\le h(x^*) + \frac{1}{2\beta} \|x^* -x_{k(j)-1} \|^2 + \inner{x^* -x_{k(j)} }{\nabla \Phi(x_{k(j)-1})} \nonumber\\
&\quad+ \|x^* -x_{k(j)} \|\|\widehat{\nabla} \Phi(x_{k(j)-1})- \nabla \Phi(x_{k(j)-1})\| \nonumber\\
&\le h(x^*) + \frac{1}{2\beta} \|x^* -x_{k(j)-1} \|^2 + \inner{x^* -x_{k(j)} }{\nabla \Phi(x_{k(j)-1})} \nonumber\\
&\quad+ \sqrt{\Gamma}\|x^* -x_{k(j)} \|\|y_{k(j)} - y^*(x_{k(j)-1})\|. \nonumber
\end{align}
Taking limsup on both sides of the above inequality and noting that $\{x_k\}_k$ is bounded, $\nabla \Phi$ is Lipschitz, $\|x_{k+1} - x_k\| \to 0$, $x_{k(j)} \overset{j}{\to} x^*$ and $\|y_{k(j)} -y^*(x_{k(j)-1})\|\overset{j}{\to} 0$, we conclude that $\lim\sup_j h(x_{k(j)}) \le  h(x^*)$.
Since $h$ is lower-semicontinuous, we know that $\lim\inf_j h(x_{k(j)}) \ge  h(x^*)$. Combining these two inequalities yields that $\lim_j h(x_{k(j)}) = h(x^*)$. By continuity of $\Phi$, we further conclude that $\lim_j (\Phi+h)(x_{k(j)}) = (\Phi+g)(x^*)$. Since we have shown that the entire sequence $\{(\Phi+h)(x_k)\}_k$ converges to a certain finite limit $H^*$, we conclude that $(\Phi+h)(x^*) \equiv H^*$ for all the limit points $x^*$ of $\{ x_k\}_k$. {This proves the item 3.}

Finally, we prove the item 4. To this end, we have shown that for every subsequence $x_{k(j)} \overset{j}{\to} x^*$, we have that $(\Phi+h)(x_{k(j)}) \overset{j}{\to} H^*{=(\Phi+h)(x^*)}$ and there exists $u_k \in \partial (\Phi+ h)(x_{k})$ such that $u_k\to \zero$ (by \cref{eq: 5}). 
Recall the definition of limiting sub-differential, we conclude that every limit point $x^*$ of $\{x_k\}_k$ is a critical point of $(\Phi+h)(x)$, i.e., $\zero \in \partial (\Phi+ h)(x^*)$.
\end{proof}

\section{Proof of \Cref{coro_alg1}}
\coroalg*
{
\begin{proof}
\begin{align}
	\|G(x_{k+1})\|=&\frac{1}{\beta}\|x_{k+1}-\prox{\beta h} (x_{k+1}- \beta \nabla \Phi(x_{k+1}))\|\nonumber\\
	&\overset{(i)}{\le} \frac{1}{\beta} \big\|x_{k+1}-x_k+\beta\big(\nabla\Phi(x_{k+1}) - \widehat{\nabla}\Phi(x_k)\big)\big\|\nonumber\\
	&\le \frac{1}{\beta} \big\|x_{k+1}-x_k\big\| +\big\|\nabla\Phi(x_{k+1}) - \nabla\Phi(x_k)\big\| +\big\|\nabla\Phi(x_k) - \widehat{\nabla}\Phi(x_k)\big\|\nonumber\\
	&\overset{(ii)}{\le} \Big(\frac{1}{\beta}+L_{\Phi}\Big)\|x_{k+1}-x_k\| + \sqrt{\Gamma}\|y_{k+1}-y^*(x_k)\| \nonumber\\
	&\overset{(iii)}{\le} \frac{2}{\beta}\|x_{k+1}-x_k\| + \sqrt{\Gamma}\|y_{k+1}-y^*(x_k)\| \nonumber
\end{align}
where (i) uses $x_{k+1}\in\prox{\beta h}\big(x_k - \beta\widehat{\nabla} \Phi(x_k)\big)$ and the non-expansiveness of proximal mapping since $h$ is convex, (ii) uses the property that $y^*$ is $\kappa$-Lipschitz continuous, and (iii) uses the stepsize $\beta\le \frac{1}{2}(L_{\Phi}+\Gamma+\kappa^2)^{-1}$. Hence, we have

\begin{align}
	&\sum_{k=0}^{K-1}\|G(x_{k+1})\|^2\nonumber\\
	&\le 2\sum_{k=0}^{K-1} \Big(\frac{4}{\beta^2}\|x_{k+1}-x_k\|^2 + \Gamma\|y_{k+1}-y^*(x_k)\|^2\Big) \nonumber\\
	&\le \max\Big(\frac{32}{\beta}, \frac{\Gamma}{16}\Big)\sum_{k=0}^{K-1} \Big(\frac{1}{4\beta} \|x_{k+1}-x_k\|^2 +  \frac{1}{8} \big(\|y_{k+1}-y^*(x_k)\|^2 + \|y_{k+2}-y^*(x_{k+1})\|^2 \big)\Big) \nonumber\\
    &\overset{(i)}{\le} \frac{32}{\beta}\big(H(x_0) - \inf_x (\Phi+g)(x)\big), \nonumber
\end{align}
where (i) uses eq. \eqref{eq: sum_dH} and the stepsize $\beta\le \frac{1}{2}(L_{\Phi}+\Gamma+\kappa^2)^{-1}$ which implies that $\frac{32}{\beta}\ge \frac{\Gamma}{16}$. Hence, 
\begin{align}
	&\min_{0\le k\le K}\|G(x_k)\| \le\sqrt{\frac{1}{K}\sum_{k=0}^{K-1}\|G(x_{t+1})\|^2} \le \sqrt{\frac{32}{K\beta}\big(H(x_0) - \inf_x (\Phi+g)(x)\big)}. \nonumber
\end{align}
To achieve $\min_{0\le k\le K}\|G(x_k)\|\le \epsilon$, it sufficies that $K\ge \frac{32}{\beta\epsilon^2}\big(H(x_0) - \inf_x (\Phi+g)(x)\big)=\mathcal{O}(\kappa^3\epsilon^{-2})$ (the maximum possible stepsize $\beta=\frac{1}{2}(L_{\Phi}+\Gamma+\kappa^2)^{-1}=\mathcal{O}(\kappa^{-3})$). Since each inner loop and each outer loop of Algorithm \ref{alg:main_deter} involves less than 7 evaluations of gradients, Hessian-vector products and proximal mappings in total, the computational complexity takes the order of $\mathcal{O}(KT)=\mathcal{O}(\kappa^{3.5}(\ln\kappa)\epsilon^{-2})$. 
\end{proof}
}

\section{Auxiliary Lemma for Proving \Cref{thm:objrate}}
{
We first inspect the mapping $y^T(x,y)$ defined as $T$ Nesterov's accelerated gradient descent steps for minimizing $g(x,\cdot)$ with initial point $y$. Define the gradient descent operator $G_x(y)=y-\alpha \nabla_2 g(x,y)$. Note that $g(x, \cdot)$ is $L$-smooth and $\mu$-strongly convex, and our learning rate $\alpha=\frac{1}{L}\le \frac{2}{L+\mu}$. Hence, based on Lemma 3.6 in \cite{hardt2016train}, $G_x(\cdot)$ is a contraction mapping with Lipschitz constant $1-\frac{\alpha L\mu}{L+\mu}=\frac{\kappa}{\kappa+1}$. Also, it can be easily seen that $G_x(y)$ is 1-Lipschitz as a function of $x$ since $\|G_{x'}(y)-G_x(y')\|=\alpha\|\nabla g(x',y)-\nabla_2 g(x,y)\|\le \alpha L\|x'-x\|=\|x'-x\|$. With the operator $G_x$, the mapping $y^t$ can be recursively defined as follows.
\begin{align}
    &y^{0}(x,y)=y; \label{eq: y0_def}  \\
    &y^{1}(x,y)=G_x(y); \label{eq: y1_def}  \\
    &y^{t}(x,y)=(1+\eta)G_x(y^{t-1}(x,y))-\eta G_x(y^{t-2}(x,y)); t\ge 2. \label{eq: yt_def}
\end{align}

We can prove the above mapping $y^{t}$ satisfies following lemma.
\begin{lemma}\label{lemma: yt_lip}
    Under Assumptions \ref{assum:geo} \& \ref{assum:lip}, $y^T(\cdot,\cdot)$ is a $(2.5^{T+1}-1.5)$-Lipschitz continuous mapping, that is, for any two points $z:=(x,y)$ and $z':=(x',y')$,
    \begin{align}
        \|y^T(z')-y^T(z)\|\le (2.5^{T+1}-1.5)\|z'-z\|. \nonumber
    \end{align}
\end{lemma}

\begin{proof}
We will prove this Lemma by induction.

Based on eq. \eqref{eq: y0_def}, $y^0$ is $1$-Lipschitz, so this Lemma holds for $T=0$.

Based on eq. \eqref{eq: y1_def}, the following inequality holds, which implies that this Lemma also holds for $T\ge 1$
\begin{align}
    \|y^1(z')-y^1(z)\|\le& \|G_{x'}(y')-G_x(y')\|+\|G_x(y')-G_x(y)\| \nonumber\\
    \le& \|x'-x\|+\frac{\kappa}{\kappa+1}\|y'-y\| \le \sqrt{2}\|z'-z\|
\end{align}

Suppose this Lemma holds for any $T\le t-1$ ($t\ge 2$). Then, based on eq. \eqref{eq: yt_def}, 
\begin{align}
    &\|y^{t}(z')-y^{t}(z)\| \nonumber\\
    &\le (1+\eta)\|G_{x'}(y^{t-1}(z'))-G_x(y^{t-1}(z))\| + \eta\|G_{x'}(y^{t-2}(z'))-G_x(y^{t-2}(z))\| \nonumber\\
    &\le (1+\eta)\|G_{x'}(y^{t-1}(z'))-G_x(y^{t-1}(z'))\| + (1+\eta)\|G_x(y^{t-1}(z'))-G_x(y^{t-1}(z))\| \nonumber\\
    &\quad + \eta\|G_{x'}(y^{t-2}(z'))-G_x(y^{t-2}(z'))\| + \eta\|G_x(y^{t-2}(z'))-G_x(y^{t-2}(z))\| \nonumber\\
    &\le (1+\eta)\|x'-x\|+(1+\eta)\frac{\kappa}{\kappa+1}\|y^{t-1}(z')-y^{t-1}(z)\|\nonumber\\
    &\quad + \eta\|x'-x\| + \eta\frac{\kappa}{\kappa+1}\|y^{t-2}(z')-y^{t-2}(z)\| \nonumber\\
    &\stackrel{(i)}{\le} 3\|z'-z\|+2(2.5^t-1.5)\|z'-z\|+(2.5^{t-1}-1.5)\|z'-z\| \nonumber\\
    &\le \big(2.4(2.5^t)+3-1.5\big)\|z'-z\| \le (2.5^{t+1}-1.5)\|z'-z\|.
\end{align}
where (i) uses $\eta=\frac{\sqrt{\kappa}-1}{\sqrt{\kappa}+1}\le 1$, $\|x'-x\|\le\|z'-z\|$ and the assumption that $y^{T}(\cdot,\cdot)$ is $2.5^{T-1}-1.5$-Lipschitz for any $T\le t-1$. Hence, this Lemma also holds for $T=t$ and thus for all $T\in\mathbb{N}$.
\end{proof}

To prove \Cref{thm:objrate} under \KL geometry, we obtain the following bound on the sub-differential of the potential function $H$. Throughout, we denote $z:=(x,y)$, $z':=(x',y')$ and $z_k:=(x_k,y_k)$.
\begin{lemma}\label{lemma: subgradH}
    Let Assumptions~\ref{assum:geo}, \ref{assum:lip} and \ref{assum: H} hold and consider the potential function defined in \cref{eq: lyapunov}. Then, under the same choices of hyper-parameters as those of \Cref{lemma: lyapunov}, the sub-differential of $H$ satisfies the following bound:
    {\begin{align}
    \dist_{\partial_z H(z_{k+1})}(\zero) 
    &\le {\frac{2}{\beta}} \|x_{k+1} - x_{k}\|  + \sqrt{\Gamma} \|y_{k+1}-y^*(x_k)\| + 2\big(2.5^{T+1}+\kappa\big) \|y_{k+2}-y^*(x_{k+1})\|. \nonumber
    \end{align}}
\end{lemma}

\begin{proof}
Recall the potential function $H(z):= \Phi(x) +h(x) +\frac{7}{8}\|y^T(x,y)-y^*(x)\|^2$ and that $\|y^T(z)-y^*(x)\|^2$ has a non-empty subdifferential $\partial (\|y^T(z)-y^*(x)\|^2)$. By the subdifferetial rule we have 
\begin{align}
    \partial_z H(z) &\supset \partial_z(\Phi+ h)(x) + \frac{7}{8} \partial_z (\|y^T(x,y)-y^*(x)\|^2) \nonumber\\
    &=\partial(\Phi+ h)(x)\times \{\zero\} + \frac{7}{8} \partial_z (\|y^T(x,y)-y^*(x)\|^2),
\end{align}
where the second ``$=$'' uses $\partial_y(\Phi+ h)(x)=\{\zero\}$. 

Next, we derive an upper bound for the subdifferentials $\partial_{z} (\|y^T(z)-y^*(x)\|^2)$. Take any Frechet subdifferential $u\in \widehat{\partial}_z (\|y^T(z)-y^*(x)\|^2)$, we obtain from its definition that 
\begin{align}
	0&\le \liminf_{z'\neq z, z'\to z} \frac{\|y^T(z')-y^*(x')\|^2 - \|y^T(z)-y^*(x)\|^2 - u^\top (z'-z)}{\|z'-z\|} \nonumber\\
	&= \liminf_{z'\neq z, z'\to z} \frac{[y^T(z')-y^*(x') + y^T(z)-y^*(x)]^\top [y^T(z')-y^*(x') - y^T(z)+y^*(x)] - u^\top (z'-z)}{\|z'-z\|} \nonumber\\
	&\le \liminf_{z'\neq z, z'\to z} \Big[\frac{\big(\|y^T(z')-y^*(x')\| + \|y^T(z)-y^*(x)\|\big) \big(\|y^T(z') - y^T(z)\| + \|y^*(x) - y^*(x')\|\big)}{\|z'-z\|}\nonumber\\
	&\quad-\frac{u^\top (z'-z)}{\|z'-z\|}\Big]\nonumber\\
	&\overset{(i)}{\le} \liminf_{z'\neq z, z'\to z} \Big[\big(\|y^T(z')-y^*(x')\| + \|y^T(z)-y^*(x)\|\big) \Big(2.5^{T+1}+\kappa\frac{\|x'-x\|}{\|z'-z\|}\Big) - \frac{u^\top (z'-z)}{\|z'-z\|}\Big]\nonumber\\
	&\overset{(ii)}{\le} 2\big(2.5^{T+1}+\kappa\big)\|y^T(z)-y^*(x)\| - \limsup_{z'\neq z, z'\to z} \frac{u^\top (z'-z)}{\|z'-z\|}\nonumber\\
	&\overset{(iii)}{=} 2\big(2.5^{T+1}+\kappa\big)\|y^T(z)-y^*(x)\| - \|u\|\nonumber,
\end{align}
where (i) uses Lemma \ref{lemma: yt_lip} that $y^T(\cdot,\cdot)$ is a $(2.5^{T+1}-1.5)$-Lipschitz continuous mapping, (ii) uses $\kappa\ge 1$, $2\sqrt{2}\le 3$ and $\|x'-x\|\le\|z'-z\|$, and the equality in (iii) is achieved by letting $z'=z+\sigma u$ with $\sigma\to 0^+$. Hence, we conclude that $\|u\|\le 8\kappa \|y^T(x,y)-y^*(x)\|$. Since $\partial_z (\|y^T(x,y)-y^*(x)\|^2)$ is the graphical closure of $\widehat{\partial}_z (\|y^T(z)-y^*(x)\|^2)$, we have that 
\begin{align}
    \dist_{\partial_z (\|y^T(z)-y^*(x)\|^2)}(\zero) \le 2\big(2.5^{T+1}+\kappa\big) \|y^T(z)-y^*(x)\|. \label{eq:dy2_diffbound}
\end{align}

Next, using the subdifferential decomposition \eqref{eq: subgrad}, we obtain that
\begin{align}
    &\dist_{\partial_z H(z_{k+1})}(\zero) \nonumber\\
    &\le \dist_{\partial (\Phi+ h)(x_{k+1})}(\zero) + \frac{7}{8} \dist_{\partial_z (\|y^T(z_{k+1})-y^*(x_{k+1})\|^2)}(\zero) \nonumber\\
    &\overset{(i)}{\le} {\Big(\frac{1}{\beta}+L_{\Phi}\Big)} \|x_{k} - x_{k+1}\| + \sqrt{\Gamma}\|y_{k+1} - y^*(x_k)\| + 2\big(2.5^{T+1}+\kappa\big) \|y^T(z_{k+1})-y^*(x_{k+1})\|\nonumber\\
    &\overset{(ii)}{\le} \frac{2}{\beta} \|x_{k} - x_{k+1}\| + \sqrt{\Gamma}\|y_{k+1} - y^*(x_k)\| + 2\big(2.5^{T+1}+\kappa\big) \|y_{k+2}-y^*(x_{k+1})\| ,\label{eq: subgrad}
\end{align}
where (i) uses \cref{eq: dPhih}\&\eqref{eq:dy2_diffbound}, (ii) uses the hyperparameter choice that $\beta \le {\frac{1}{2}(L_{\Phi}+\Gamma+\kappa^2)^{-1}}$ in Proposition \ref{lemma: lyapunov}.
\end{proof}
}

\section{Proof of \Cref{thm:objrate}} 
\theoremobjrate*
\begin{proof}
Recall that we have shown in the proof of \Cref{thm: 1} that: 1) $\{H(x_k,y_k)\}_k$ decreases monotonically to the finite limit $H^*$; 2) for any limit point $x^*$ of $\{x_k\}_k$, $H(x^*)$ has the constant value $H^*$. Hence, the \KL inequality holds after a sufficiently large number of iterations, i.e., there exists $k_0\in \mathbb{N}^+$ such that the following holds for all $k\ge k_0$.
\begin{align}
	\varphi'(H({x_k,y_k}) -H^*) \dist_{{\partial_z H(x_k,y_k)}}(\zero)\ge 1. \nonumber
\end{align}
Rearranging the above inequality and utilizing \cref{eq: subgrad}, we obtain that for all $k\ge k_0$, 
\begin{align}
    &\varphi'(H(x_k,y_k) -H^*) \nonumber\\
    &\ge \frac{1}{\dist_{\partial_z H(x_k,y_k)}(\zero)} \nonumber\\
    &\ge {\Big(\frac{2}{\beta} \|x_{k-1} - x_{{k}}\| + \sqrt{\Gamma}\|y_k - y^*(x_{k-1})\| + 2\big(2.5^{T+1}+\kappa\big) \|y_{k+1}-y^*(x_k)\|\Big)^{-1}}. \label{eq: 7}
\end{align}

For simplicity, denote $d_k:=H(x_k,y_k)-H^*$ as the function value gap. Then, for a sufficiently large $k$ such that \cref{eq: 7} holds, we have
\begin{align}
    &c^{-2}d_k^{2(1-\theta)}\nonumber\\
    &\stackrel{(i)}{=}\big[\varphi'(d_k)\big]^{-2} \nonumber\\
    &\stackrel{(ii)}{\le} \Big(\frac{2}{\beta} \|x_{k-1} - x_{{k}}\| + \sqrt{\Gamma}\|y_k - y^*(x_{k-1})\| + 2\big(2.5^{T+1}+\kappa\big) \|y_{k+1}-y^*(x_k)\|\Big)^2 \nonumber\\
    &\stackrel{(iii)}{\le} \frac{12}{\beta^2} \|x_{k-1} - x_{k}\|^2 + 3\Gamma\|y_k - y^*(x_{k-1})\|^2 + 24(5^{T+1}+\kappa^2) \|y_{k+1}-y^*(x_k)\|^2\nonumber\\
    &\le \max\Big(\frac{48}{\beta},24\Gamma,24(5^{T+1}+\kappa^2)\Big)\Big(\frac{1}{4\beta} \|x_{k-1} - x_{k}\|^2 + \frac{1}{8}\big(\|y_k-y^*(x_{k-1})\|^2 + \|y_{k+1}-y^*(x_k)\|^2 \big)\Big) \label{eq: dphi_power}\\
    &\stackrel{(iv)}{\le} \max\Big(\frac{48}{\beta},24\Gamma,24(5^{T+1}+\kappa^2)\Big) \big(H(x_{k-1},y_{k-1})-H(x_k,y_k)\big)\nonumber\\
    &\le \max\Big(\frac{48}{\beta},24\Gamma,24(5^{T+1}+\kappa^2)\Big) \big(d_{k-1}-d_{k}\big), \nonumber
\end{align}
where (i) uses the equality that $\varphi'(s)=cs^{\theta-1}$ based on \Cref{def: KL}, (ii) uses \cref{eq: 7}, (iii) uses the inequality that $(a+b+c)^2\le 3a^2+3b^2+3c^2$, and (iv) uses \Cref{lemma: lyapunov}. Rearranging the above inequality yields that
\begin{align}
    d_{k-1}\ge d_{k}+Cd_k^{2(1-\theta)},\label{eq: dphi_bound}
\end{align}
where $C:=\big[c\max\big(\frac{48}{\beta},24\Gamma,24(5^{T+1}+\kappa^2)\big)\big]^{-1}>0$ is a constant. 

Next, we prove the convergence rates case by case.


(Case I) If $\theta\in\big(\frac{1}{2},1\big)$, then since $d_{k}\ge 0$, \cref{eq: dphi_bound} implies that $d_{k-1}\ge Cd_k^{2(1-\theta)}$, which is equivalent to that
\begin{align}
    C^{-\frac{1}{2\theta-1}}d_{k} \le \Big(C^{-\frac{1}{2\theta-1}}d_{k-1}\Big)^{\frac{1}{2(1-\theta)}}. \nonumber
\end{align}
Since $d_k\downarrow 0$, $C^{-\frac{1}{2\theta-1}}d_{k_0}\le e^{-1}$ for sufficiently large $k_0\in\mathbb{N}^+$. Hence, the above inequality implies that for $k\ge k_0$,
\begin{align}
    C^{-\frac{1}{2\theta-1}}d_{k} \le \Big(C^{-\frac{1}{2\theta-1}}d_{k_0}\Big)^{\Big[\frac{1}{2(1-\theta)}\Big]^{k-k_0}} \le \exp\Big(-\Big[\frac{1}{2(1-\theta)}\Big]^{k-k_0}\Big).
\end{align}
Since $\theta\in\big(\frac{1}{2},1\big)$ implies that $\frac{1}{2(1-\theta)}>1$, the above inequality implies that $d_k\downarrow 0$ (i.e. $H(x_k,y_k)\downarrow H^*$) at the super-linear rate given by \cref{eq: superlinear_converge}. 

(Case II) If $\theta=\frac{1}{2}$, then \cref{eq: dphi_bound} implies that
\begin{align}
    d_{k}\le (1+C)^{-1}d_{k-1},\nonumber
\end{align}
which further implies that $d_k\downarrow 0$ (i.e. $H(x_k,y_k)\downarrow H^*$) at the linear rate given by \cref{eq: linear_converge}.

(Case III) If $\theta\in\big(0,\frac{1}{2}\big)$, then {denote $\psi(s)=\frac{1}{1-2\theta}s^{-(1-2\theta)}$} and consider the following two subcases.

If $d_{k-1}\le 2d_k$, then 
\begin{align}
    \psi(d_k)-\psi(d_{k-1})=&\int_{d_k}^{d_{k-1}} -\psi'(s)ds=\int_{d_k}^{d_{k-1}} s^{-2(1-\theta)} ds \stackrel{(i)}{\ge} d_{k-1}^{-2(1-\theta)}(d_{k-1}-d_k) \nonumber\\
    \stackrel{(ii)}{\ge}& C\Big(\frac{d_k}{d_{k-1}}\Big)^{2(1-\theta)} \stackrel{(iii)}{\ge} 2^{-2(1-\theta)}C \nonumber
\end{align}
where (i) uses $d_k\le d_{k-1}$ and $-2(1-\theta)<-1$, (ii) uses \cref{eq: dphi_bound}, and (iii) uses $C>0$, $d_{k-1}\le 2d_k$ and $2(1-\theta)>1$. 

If $d_{k-1}> 2d_k$, then for $k\ge k_0$
\begin{align}
    \psi(d_k)-\psi(d_{k-1})=&\frac{1}{1-2\theta}\big(d_k^{-(1-2\theta)}-d_{k-1}^{-(1-2\theta)}\big)\stackrel{(i)}{\ge} \frac{1}{1-2\theta}\big(d_k^{-(1-2\theta)}-(2d_k)^{-(1-2\theta)}\big)\nonumber\\
    \ge&\frac{1-2^{-(1-2\theta)}}{1-2\theta}d_k^{-(1-2\theta)} \stackrel{(ii)}{\ge} \frac{1-2^{-(1-2\theta)}}{1-2\theta}d_{k_0}^{-(1-2\theta)} \nonumber
\end{align}
where (i) uses $d_{k-1}> 2d_k$ and $-(1-2\theta)<0$, and (ii) uses $-(1-2\theta)<0$, $\frac{1-2^{-(1-2\theta)}}{1-2\theta}>0$ and $d_k\le d_{k_0}$. 

Combining the above two subcases yields that
\begin{align}
    \psi(d_k)-\psi(d_{k-1})\ge \min\Big[2^{-2(1-\theta)}C, \frac{1-2^{-(1-2\theta)}}{1-2\theta}d_{k_0}^{-(1-2\theta)} \Big]=\frac{U}{1-2\theta}>0, k\ge k_0 \nonumber
\end{align}
where $U:=\min\Big( 2^{-2(1-\theta)}C(1-2\theta), \big(1-2^{-(1-2\theta)}\big)d_{k_0}^{-(1-2\theta)} \Big)>0$. Iterating the above inequality yields that 
\begin{align*}
    \psi(d_k)\ge \psi(d_{k_0})+\frac{U}{1-2\theta}(k-k_0) \ge \frac{U}{1-2\theta}(k-k_0)
\end{align*}
Then by substituting $\psi(s)=\frac{1}{1-2\theta}s^{-(1-2\theta)}$, the inequality above implies that
that $d_k\downarrow 0$ (i.e. $H(x_k,y_k)\downarrow H^*$) at the sub-linear rate given by \cref{eq: sublinear_converge}.
\end{proof}

\end{document}